\documentclass[11pt]{article}
\usepackage{fullpage}
\usepackage{amsmath,amsthm,amsfonts,amssymb,dsfont,mathtools}
\usepackage{hyperref}
\hypersetup{colorlinks,
            linkcolor=blue,
            citecolor=blue,
            urlcolor=magenta,
            linktocpage,
            plainpages=false}
\usepackage{arxiv_style}
\usepackage{microtype}
\title{Near-Optimal SQ Lower Bounds for Agnostically Learning Intersections of Halfspaces}

\title{ \protect Near-Optimal Statistical Query Lower Bounds for Agnostically Learning Intersections of Halfspaces with Gaussian Marginals}

\author{Daniel Hsu\thanks{Columbia University, \url{djhsu@cs.columbia.edu}.  Supported by NSF grants CCF-1740833,  IIS-1563785,  and a Sloan Re-search Fellowship.}, Clayton Sanford\thanks{Columbia University, \url{clayton@cs.columbia.edu}.  Supported by an NSF GRFP fellowship, NSF grant CCF-1563155, and DH's Google  Faculty  Research  Award. }, Rocco A.~Servedio\thanks{Columbia University, \url{rocco@cs.columbia.edu}.  Supported by NSF grants CCF-1814873, IIS-1838154, CCF-1563155, and by the Simons Collaboration on Algorithms and Geometry.},  Emmanouil V. Vlatakis-Gkaragkounis\thanks{Columbia University, \url{emvlatakis@cs.columbia.edu}.  Supported by NSF grants CCF-1703925, CCF-1763970,CCF-1814873, CCF-1563155, and by the Simons Collaborationon Algorithms and Geometry.}}

\begin{document}

\maketitle

\begin{abstract}%

We consider the well-studied problem of learning intersections of halfspaces under the Gaussian distribution in the challenging \emph{agnostic learning} model. 
Recent work of \citet{dkpz21} shows that any Statistical Query (SQ) algorithm for agnostically learning the class of intersections of $k$ halfspaces over $\R^n$ to constant excess error \ignore{error $\OPT+\eps$}either must make queries of tolerance at most \smash{$n^{-\tilde{\Omega}(\sqrt{\log k})}$}\ignore{$n^{-(\tilde{\Omega}(\sqrt{\log k}/\eps) + \Omega(1/\eps^2))}$} or must make \smash{$2^{n^{\Omega(1)}}$} queries. 
We strengthen this result by improving the tolerance requirement to \smash{$n^{-\tilde{\Omega}(\log k)}$}.
This lower bound is essentially best possible since an SQ algorithm of \citet{kos08} agnostically learns this class to any constant excess error using \smash{$n^{O(\log k)}$} queries of tolerance \smash{$n^{-O(\log k)}$}.
We prove two variants of our lower bound,
each of which combines ingredients from \citet{dkpz21} with (an extension of) a different earlier approach for agnostic SQ lower bounds for the Boolean setting due to \citet{dftww14}. Our approach also yields lower bounds for agnostically SQ learning the class of ``convex subspace juntas'' (studied by \citealp{vempala10}) and the class of sets with bounded Gaussian surface area; all of these lower bounds are nearly optimal since they essentially match known upper bounds from \citet{kos08}. 

\end{abstract}

\newpage


\section{Introduction}\label{sec:intro}

Linear threshold functions, or \emph{halfspaces}, are ubiquitous in machine learning.
They arise in the context of many statistical models for classification~\citep{duda1973pattern}, and they are the focus of many well-known machine learning methods, including Perceptron~\citep{rosenblatt1962principles}, Support Vector Machines~\citep{vapnik1982estimation}, and AdaBoost~\citep{freund1997decision}.
In this work, we consider the problem of agnostic learning for a natural and well-studied generalization of this function class: \emph{intersections of halfspaces}.

Although many efficient algorithms for learning halfspaces have been developed to handle a wide variety of settings~\citep{blumer1989learnability,blum1998polynomial,kkms,awasthi2017power,diakonikolas2021efficiently}, known algorithms for intersections of halfspaces are conspicuously limited in scope and applicability.
Indeed, no efficient PAC learning algorithms are known even for the case of intersections of two halfspaces.
There, a learner faces a ``credit assignment'' problem when considering negative examples, as either of the two halfspaces may be responsible for an example being classified as negative, but the learner is not privy to this information.
This prevents a straightforward formulation of the learning problem as a linear program, which had sufficed in the case of learning single halfspaces.

Because of the apparent difficulty of going beyond single halfspaces, much of the progress has come from learning under ``nice'' data marginal distributions, such as the uniform distribution or the Gaussian distribution \citep{bk97,vempala97,vempala2010random,kos04,kkms,kos08,vempala10,kane14}.
The fastest algorithm to date for agnostically learning intersections of halfspaces under Gaussian marginals in $\R^n$ is $L^1$ polynomial regression~\citep{kkms}, which was shown by \citet{kos08} to successfully learn up to any constant excess error in time $n^{O(\log k)}$.
(Under the additional assumption of realizability, \citet{vempala10} showed that when $k = o(n)$, preprocessing with principal component analysis improves this running time to $\poly(n,k) + k^{O(\log k)}$.)
Since this upper bound has resisted improvement for several years, attention has turned to trying to prove lower bounds, and such lower bounds are the subject of this paper.

The Statistical Query (SQ) model of \citet{kearns98} offers an attractive setting for proving unconditional lower bounds against a broad class of learning algorithms.
SQ learning algorithms can access data only through imperfect estimates of the expected values of query functions with respect to the data distribution.
Nearly all known learning algorithms, including those of \citet{kkms}, \citet{kos08} and \citet{vempala10}, can be implemented within the SQ model, so lower bounds in the SQ model are evidence for the computational difficulty of a learning problem.
Because these algorithmic results for agnostic learning hold only under ``nice'' marginal distributions, it is of interest to prove \emph{distribution-dependent} SQ lower bounds under the same marginals.

The pioneering work of \citet{dftww14} provided a blueprint for proving such distribution-dependent SQ lower bounds.
They proved an equivalence between the approximation resilience of functions in a concept class and the SQ agnostic learnability of that class, and used this equivalence to obtain the first super-polynomial SQ lower bounds for agnostically learning the important concept class of monotone juntas under the uniform distribution.
To establish SQ lower bounds for agnostic learning under Gaussian marginals, \citet{dkpz21} extended the approach of \citet{dftww14} using new duality arguments and embedding techniques.
In doing so, they obtained lower bounds for agnostically learning a number of Boolean concept classes (as well as some real-valued concept classes).
For intersections of $k$ halfspaces,
their agnostic SQ lower bound
is
\smash{$n^{\tilde{\Omega}(\sqrt{\log k})}$},
which should be contrasted with the \smash{$n^{O(\log k)}$} upper bound of \citet{kos08}.
In fact, they conjectured that it may be the upper bound that is loose.

Our results prove that the algorithmic results of \citet{kos08} are indeed nearly optimal. 
Specifically, we show that any SQ algorithm that agnostically learns intersections of $k \leq \exp(O(n^{0.245}))$ halfspaces to any constant excess error must have complexity at least \smash{$n^{\tilde{\Omega}(\log k)}$}.
The notion of complexity is made more precise in the informal theorem statement below.

\begin{theorem}[Informal version of \Cref{thm:alt-cube-ganzburg}]\label[theorem]{thm:informal}
	Any SQ algorithm that agnostically learns intersections of $k$ halfspaces to excess error $\eps$ under Gaussian marginals requires either $2^{n^{\Omega(1)}}$ queries or at least one query of tolerance \smash{$n^{-\tilde\Omega(\log k + 1/\eps^2)}$}.
\end{theorem}
This result is nearly optimal for any constant $\eps$, up to a $\log\log k$ factor in the exponent, because the $n^{O(\log k)}$ time and sample complexity upper bounds from \citet{kos08} can be achieved by an SQ algorithm.
We note that by the AM-GM inequality the exponent $\tilde\Omega(\log k + 1/\eps^2)$ in our lower bound is always at least $\tilde\Omega(\sqrt{\log k}/\eps)$, which is the exponent from the SQ lower bound of \citet{dkpz21}, but can also be significantly stronger.


In fact,
when $k$ is small (\smash{$k = O(n^{0.49})$}) 
we show that the hardness of learning intersections of $2k$ halfspaces
is already present in the easier problem of learning a simple subset of the class: the family of $k$-dimensional cubes.
This result, given in \Cref{thm:main-cube},
relies on new technical facts about the $L^1$-error approximation degree of cube functions under Gaussian marginals. 
Our bounds additionally imply new hardness results on learning functions with bounded Gaussian surface area and convex subspace juntas (see \Cref{thm:GSAlb,thm:junta-lb} respectively).


\subsection{Techniques} \label{sec:techniques}

Our proofs follow the blueprints of \citet{dftww14} and \citet{dkpz21} and build upon them by using weak learning lower bounds from \citet{ds21} and new technical innovations for proving resilience with respect to continuous measures.
Put roughly, \citet{dftww14}:
\begin{enumerate}[label=(\alph*), noitemsep]
\item\label{item:dftww-a} introduced a notion of \emph{approximate resilience} on the Boolean cube and established an equivalence to $L^1$ approximate degree using linear programming duality;
\item\label{item:dftww-b} used a combinatorial argument to show that if a $k$-dimensional function $f$ is approximately resilient, then there exists a family of $k$-juntas ($n$-dimensional embeddings of $f$ for $n \gg k$) that is hard to agnostically learn in the SQ model;
\item\label{item:dftww-c} used Boolean Fourier analysis to prove approximate resilience for the $\mathsf{Tribes}$ function (a monotone read-once DNF); and
\item\label{item:dftww-d} proved a tighter approximate resilience bound for other monotone Boolean functions by combining a  hardness result on weak learning of \citet{bbl98} with an agnostic learning algorithm based on $L^1$ polynomial approximation by \citet{kkms}.
\end{enumerate}
To transfer this methodology to the Gaussian measure on $\R^n$, \citet{dkpz21}: 
\begin{enumerate}[label=(\alph*$'$), noitemsep]
\item\label{item:dkpz-a} extended the equivalence of approximate resilience and $L^1$ approximate degree to Gaussian marginals with a more technical argument involving an infinite linear program and the Hahn-Banach Theorem;
\item\label{item:dkpz-b} 
showed that $L^1$ polynomial inapproximability of a $k$-dimensional function implies the hardness of SQ-learning a family of $n$-dimensional embeddings of $f$ applied to $k$-dimensional subspaces\footnote{The underlying hard problem is distinguishing a standard (multivariate) Gaussian from a distribution that differs from the standard Gaussian only in the high-order moments of a $k$-dimensional projection~\citep{dks17}.}; and 
\item\label{item:dkpz-c} lower-bounded the $L^1$ approximate degree of an intersection of $k$ halfspaces using a new connection with Gaussian noise sensitivity. 
\end{enumerate}

Our results are obtained using a hybrid of the \citet{dftww14} and \citet{dkpz21} approaches. More precisely, we
rely on \ref{item:dkpz-a} and \ref{item:dkpz-b} to establish agnostic SQ lower bounds over Gaussian marginals for approximately resilient functions, but we draw inspiration from \ref{item:dftww-c} instead of \ref{item:dkpz-c} to bound the approximate resilience of the $\cubek$ function by directly analyzing its Hermite representation.
We also draw inspiration from \ref{item:dftww-d} when we lower-bound the approximate resilience of other intersections of halfspaces by using a recent hardness result from \citet{ds21} for weak learning those functions.

In more detail, \Cref{thm:main-cube} proves the hardness of learning the restricted class of $k$-dimensional cubes for \smash{$k = O(n^{0.49})$} in $n$-dimensional space by directly bounding the approximate resilience of a single cube function, $\cubek: \R^k \to \R$.
That is, we show that $\cubek$ is close in $L^1$-distance to a bounded function that is orthogonal to all polynomials of degree \smash{$d = \tilde\Omega(\log k)$}.
To construct this bounded function, we develop a new argument which is inspired by \ref{item:dftww-c} but is significantly more technically involved.
Due to the unboundedness and continuity of our $\mathcal{N}(0, I_n)$ setting, our argument requires a careful iterative construction, which involves defining a thresholding transform that reduces the low-degree Hermite coefficients of its input while maintaining its boundedness and taking the limit of applying the transform an infinite number of times.
The key properties of $\cubek$ for this argument are the boundedness of its outputs and its small low-degree Hermite weight. The approximate resilience of $\cubek$ provides an almost-tight bound on the $L^1$ approximate degree of the function, and the main result follows by direct application of \ref{item:dkpz-a} and \ref{item:dkpz-b}.

\Cref{thm:alt-intersection}, which shows the hardness of learning to constant accuracy the broader classes of all intersections of $k$ halfspaces for {any $k = \exp(O(n^{0.245}))$}, instead relies on the combination of recent lower bounds on the number of queries needed to weakly learn intersections of $k$ halfspaces from \citet{ds21} and well-known algorithmic results of \citet{kkms} for agnostically learning functions with bounded $L^1$ approximate degree. This approach draws inspiration from \ref{item:dftww-d}.
We show that the $L^1$ approximate degree of a random intersection of $k$ halfspaces must be at least $\tilde\Omega(\log k)$ with high probability, since otherwise there would be a contradiction between the aforementioned works:  \citet{kkms} would provide an algorithm to weakly learn intersections of halfspaces using fewer queries than the lower bound established by \citet{ds21}.
As before, these bounds on polynomial inapproximability translate to SQ learning lower bounds via the machinery of \citet{dkpz21}.

All of the above arguments are for constant excess error (constant $\eps$). We introduce the dependence on \smash{$\frac1\eps$} in \Cref{thm:alt-cube-ganzburg} by augmenting the previously-considered intersections of $k$ halfspaces with a single halfspace (in an additional dimension) that passes through the origin. \citet{ganzburg02} showed that a single halfspace has $L^1$ $\eps$-approximate degree \smash{$\Omega(\frac1{\eps^2})$}, and we use this to show that our new intersection of halfspaces has approximate degree \smash{$\tilde\Omega(\log k  + \frac1{\eps^2})$}.

\subsection{Related work} \label{sec:related}

Efficient algorithms are known for PAC learning intersections of halfspaces under certain marginal distributions.
\citet{baum1990polynomial} gave an algorithm for learning two homogeneous halfspaces under origin-symmetric distributions, and the same algorithm is now known to also succeed under mean-zero log-concave distributions~\citep{klivans2009baum}.
For PAC learning intersections (and other functions) of $k$ general halfspaces, algorithms are known for the uniform distribution on the unit ball \citep{bk97}, the uniform distribution on the Boolean cube \citep{kos04,kkms,kane14}, Gaussian distributions~\citep{kos08,vempala10}, and general log-concave distributions~\citep{vempala97,vempala2010random}.
In most of these cases, the dependence on $k$ in the running time is \smash{$n^{\Omega(k)}$} or worse
(the exceptions are the algorithms for Gaussian or uniform on $\smash{\bn}$ marginals).
In fact, only the $L^1$ polynomial regression algorithm is known to succeed in the agnostic setting, and only under Gaussian or uniform on $\bn$ marginals ~\citep{kos08,kane14}.
Finally, efficient algorithms are also known for PAC learning intersections of any constant number of halfspaces under marginals satisfying a geometric margin condition~\citep{arriaga2006algorithmic,klivans2008learning}, and also for learning intersections (and other functions) of halfspaces using membership queries~\citep{kwek1998pac,gopalan2012learning}.

Our work focuses on hardness of learning intersections of halfspaces.
Besides the SQ lower bounds of \citet{dkpz21} for (agnostic) learning under Gaussian marginals (which built on the closely related work of \citet{dftww14}), there is other evidence for the difficulty of this learning problem.
First, distribution-free PAC learning---both proper learning and improper learning with certain hypothesis classes---is known to be NP-hard~\citep{blum1992training,megiddo1988complexity}, and lower bounds on the threshold degree of intersections of two halfspaces due to \citet{sherstov2013optimal} rule out efficient algorithms that use polynomial threshold functions as hypotheses.
Cryptographic lower bounds~\citep{ks06} give further evidence that distribution-free PAC learning is  hard even if the learner is permitted to output any polynomial-time computable hypothesis.
(The distribution-free correlational SQ lower bounds of \citet{gkk20} give similar evidence for restricted types of learners.)
These lower bounds leave open the possibility that fixed-distribution PAC learning is tractable, but again there is evidence against this, at least for certain classes of learning algorithms.
\citet{ks07} showed that there is a (non-uniform) marginal distribution on the Boolean cube under which the SQ dimension of intersections of $\sqrt{n}$ halfspaces is at least $\smash{2^{\Omega(\sqrt{n})}}$; this implies lower bounds for (weak) SQ learning under that distribution.
Finally, \citet{kos08} gave membership query lower bounds for learning certain convex bodies under Gaussian marginals.
These lower bounds are exhibited by intersections of $k$ halfspaces for sufficiently large $k$, but they do not rule out $\poly(n)$ query algorithms unless $k$ is at least polynomially large in $n$.
Moreover, these lower bounds are insensitive to the error parameter $\epsilon$ sought by the learner, and in particular do not become higher for subconstant $\epsilon$.

\subsection{Organization} In \Cref{sec:lb-resilience} we prove \Cref{thm:main-cube}, which 
 gives an $\smash{n^{\tilde{\Omega}(\log k)}}$ SQ lower bound for agnostically learning intersections of $k$ halfspaces (in fact, $k$-dimensional cubes) to constant excess error when $\smash{k=O(n^{0.49}).}$
 \Cref{sec:lb-weak-learn} gives a similar SQ lower bound for larger values of $k$ (using different arguments and less structured intersections of halfspaces which are not cubes).
\Cref{sec:ganzburg} improves the quantitative results of both these sections by allowing for subconstant excess error, thereby establishing \Cref{thm:informal} (see \Cref{thm:alt-cube-ganzburg} in \Cref{sec:ganzburg} for a detailed theorem statement). \Cref{sec:discussion} extends our results to the concept class of functions with bounded Gaussian surface area and convex subspace juntas, and gives some observations on lower bounds for $L^1$ polynomial approximation.

\section{Preliminaries}\label{sec:prelims}

\subsection{Functions in Gaussian space}



For any $k \in \N$, the standard Gaussian distribution on $\R^k$ is denoted by $\mnormalk$.
For $q \geq 1$, let $\smash{\norml[q]{f} = \smash{\EEl[\bx \sim \mnormalk]{|f(\bx)|^q}^{1/q}}}$ denote the $L^q$-norm of $f \in L^q(\mnormalk)$, and let $\innerprod{f}{g} = \EEl[\bx \sim \mnormaln]{f(\bx)g(\bx)}$ denote the inner product between $f, g \in L^2(\mnormalk)$.
For a multi-index $\smash{J \in \N^k}$, let $\# J$ denote the number of nonzero elements of $J$, and let $\smash{|J| = J_1 + \cdots + J_k}$.
Let $\calP_{k,d}$ denote the family of all polynomials $p : \R^k \to \R$ of degree at most $d$.

In \Cref{ap:hermite}, we recall basic facts about the \emph{Hermite polynomials} $\{ H_J \}_{J \in \N^k}$, which form an orthogonal basis for $L^2(\mnormalk)$, as well as
some tools based on Gaussian hypercontractivity.

\subsection{Agnostic learning under Gaussian marginals and Statistical Query learning}

We recall the framework of agnostic learning under Gaussian marginals. Given a concept class ${\cal C}$ of functions from $\R^n$ to $\bits$, an agnostic learning algorithm is given access to i.i.d.~labeled examples $(\bx,\by)$ drawn from a distribution ${\cal D}$ over $\R^n \times \bits$, where the marginal of ${\cal D}$ over the first $n$ coordinates is $\mnormaln$. Intuitively, a successful agnostic learning algorithm for ${\cal C}$ is one which can find a hypothesis that correctly predicts the label $\by$ almost as well as the best predictor in ${\cal C}$. More precisely, an \emph{agnostic learning algorithm for ${\cal C}$ under Gaussian marginals with excess error $\eps$} is an algorithm which, with high probability, outputs a hypothesis function $h: \R^n \to \bits$ such that $\Pr_{(\bx,\by) \in {\cal D}}[h(\bx) \neq \by] \leq \OPT + \eps$, where $\OPT = \inf_{f \in {\cal C}} \Pr_{(\bx,\by) \in {\cal D}}[f(\bx) \neq \by].$
The special case of $\OPT=0$ corresponds to (realizable) PAC learning under the Gaussian distribution.


The above definition is for a learning scenario in which the learner has access to individual random examples. 
In the well-known Statistical Query (SQ) learning model, the learning algorithm cannot access individual  examples from ${\cal D}$ but instead has access to a ``STAT oracle.'' 
\begin{definition}
	A learning algorithm $\mathcal{A}$ has access to a \emph{STAT oracle} if $\mathcal{A}$ makes queries with a function $g: \R^n \times \flip \to [-1, 1]$ and a tolerance parameter $\tau > 0$ and recieves an estimate of the expectation $\EEl[(\bx, \by)\sim \mathcal{D}]{g(\bx, \by)}$ that is accurate up to additive error $\pm \tau$.
	An algorithm $\mathcal{A}$ with access to a STAT oracle is an \emph{SQ agnostic learning algorithm} for concept class $\mathcal{C}$ if it returns with high probability a hypothesis $h: \R^n \to \flip$ such that  $\Pr_{(\bx,\by) \in {\cal D}}[h(\bx) \neq \by] \leq \OPT + \eps$.
\end{definition}

\subsection{Resilience and $L^1$ polynomial approximation}

\citet{dftww14} established a useful connection between lower bounds for SQ agnostic learning under the uniform distribution on $\bn$ and the notion of \emph{resilience} for bounded functions.  This connection was extended to the Gaussian setting by \citet{dkpz21}, and it also plays an essential role in our results.  

Intuitively, a function is ``resilient'' if it has zero correlation with all low-degree basis functions.  More formally, we have the following:

\begin{definition}
\label[definition]{def:resilience}
A function	$g : \R^n \to [-1, 1]$ is \emph{$d$-resilient} if $\innerprod{g}{p} = 0$ for every  $p \in \calP_{n,d}$ (equivalently, $\innerprod{g}{H_J}=0$ for every $|J| \leq d$).
For $0 \leq \alpha < 1$, a function $f: \R^n \to [-1, 1]$ is said to be \emph{$\alpha$-approximately $d$-resilient} if there exists a $d$-resilient \emph{witness} $g: \R^n \to [-1,1]$ such that $\norm[1]{f - g} \leq \alpha$.
\end{definition}

Next we define the notion of $L^1$ polynomial approximation:

\begin{definition}
\label[definition]{def:L1polynomialapproximation}
Given $0 \leq \eps < 1$ and $f: \R^n \to [-1,1]$, we say that the \emph{$L^1$ $\eps$-approximate degree of $f$} is the smallest value $d \geq 0$ such that there exists a polynomial $p \in \calP_{n,d}$ satisfying
$\norm[1]{f-p} \leq \eps.$
\end{definition}

\Cref{def:L1polynomialapproximation} is of course equivalent to $d$ being the largest value such that every polynomial $p$ of degree at most $d-1$ has $\norml[1]{f-p} > \eps.$

Using linear programming duality, for the setting of functions $f: \bn \to [-1,1]$ and the uniform distribution over $\bn$, \cite{dftww14} established an \emph{equivalence} between the $L^1$-distance to the closest $d$-resilient bounded function (cf.~\Cref{def:resilience}) and the best possible accuracy of $L^1$ polynomial approximation by degree-$d$ polynomials (cf.~\Cref{def:L1polynomialapproximation}). They did this by showing (see their Theorem~1.2) that for $f: \bn \to [-1,1]$, if the former quantity is $\alpha$ then the latter quantity is $1-\alpha$.\ignore{ is $\alpha$-approximately $d$-resilient if and only if the $L^1$ polynomial $(1-\alpha)$-approximate degree of $f$ is $d$ (where the relevant notion of $L^1$ polynomial approximation is now with respect to the uniform distribution over $\bn$).} 

This equivalence was extended to the setting of Gaussian space (our domain of interest in the current work) by \citet{dkpz21}; a more involved argument is required for this setting, essentially because now the linear programming duality involves an infinitely large linear program, but the result still goes through.  The proof of their Proposition~2.1 establishes the following:\footnote{The statement of Proposition~2.1 of \citet{dkpz21} only goes in one direction (that $L^1$ polynomial approximate degree implies approximate resilience), but the proof establishes both directions.}

\begin{lem} [Equivalence of approximate resilience and $L^1$ approximate degree] \label[lemma]{lem:dkpz2.1}
A function  $f: \R^n \to [-1, 1]$ is $\alpha$-approximately $d$-resilient if and only if its $L^1$ $(1-\alpha)$-approximate degree is $d$.
\end{lem}


In the Boolean hypercube setting, \citet{dftww14} combined their $L^1$ polynomial approximation characterization of resilience with standard SQ lower bounds and standard results on the existence of combinatorial designs to show the following:  if $f: \bits^k \to \bits$ is an $\alpha$-approximately $d$-resilient function, then (roughly speaking; see their Lemma~2.1 for a precise statement) any concept class of functions from $\bn$ to $\bits$ containing all ``embeddings'' of $f$ (according to the combinatorial design) admits a Statistical Query lower bound.  

\citet{dkpz21}~carried out a similar-in-spirit argument in the setting of Gaussian space.  We note that their result is significantly technically more challenging than the analogous argument of \citet{dftww14}; it builds on recent SQ lower bounds for distinguishing distributions due to \citet{dks17}, and uses embeddings of low-dimensional functions in hidden low-dimensional subspaces rather than combinatorial designs. We record their key result below, which will be crucially used for all of our agnostic SQ lower bounds:

\begin{lem} [\citealp{dkpz21}, Theorem~1.4] \label[lemma]{thm:dkpz-sq-lb}
Let $n,m \in \N$ with $\smash{m \leq n^a}$ for any $0 < a < 1/2$, and let $\smash{\eps \geq n^{-c}}$ for a suitably small absolute constant $c>0$. Given any function $f: \R^m \to \bits$, let $d$ be the $L^1$ $(2\eps)$-approximate degree of $f$.\footnote{By \Cref{lem:dkpz2.1}, this condition is equivalent to $f$ being $(1-2\eps)$-approximately $d$-resilient.} Let ${\cal C}$ be a class of $\bits$-valued functions on $\R^n$ which includes all functions of the form $F(x) = f(Px)$ for all $P \in \R^{m \times n}$ such that \smash{$P P^\T = I_m$}. Any SQ algorithm that agnostically learns ${\cal C}$ under $\mnormaln$ to error $\OPT + \eps$ either requires queries with tolerance at most \smash{$n^{-\Omega(d)}$} or makes at least \smash{$2^{n^{\Omega(1)}}$} queries.
\end{lem}

\ignore{






}



\section{Hardness of SQ learning to constant excess error via approximate resilience}\label{sec:lb-resilience}

\newcommand\high[1]{\mathrm{High}_{#1}}
\newcommand\low[1]{\mathrm{Low}_{#1}}
\newcommand\dam[1]{\mathrm{TruncHigh}_{#1}}

The main result of this section is \Cref{thm:main-cube}, which, roughly speaking, shows that any SQ algorithm that makes a sub-exponential number of statistical queries and agnostically learns the concept class of ``embedded $k$-dimensional cubes'' (for $k = O(n^{0.49})$) to any constant excess error that is bounded below $\half$ must make queries of tolerance $\smash{n^{-\Omega(\log(k) / \log \log k)}}$. (Note that this gives a special case of \Cref{thm:informal} in which the excess error $\eps$ is constant and $k = O(n^{0.49}).$)
This is done by establishing that the $k$-dimensional cube function is approximately resilient; recall that by \Cref{def:resilience}, this means that it is close in $L^1$ distance to a bounded function that is orthogonal to all low-degree polynomials.

We define the function $\cubek: \R^k \to \flip$ as 
$\cubek(y) := \sign(\theta_k - \norm[\infty]{y})$.
(Note that this is equivalent to $\smash{\cubek(y) =  2\prod_{i=1}^k \indicator{\abs{y_i} \leq \theta_k} - 1}$.)
In words, $\cubek^{-1}(1)$ is the axis-aligned origin-centered  solid cube with side length $2 \theta_k$, where $\theta_k\geq0$ is chosen to ensure that $\EEl[\by \sim \mnormalk]{\cubek(\by)} = 0$.
Note that $\cubek$ is an intersection of $2k$ halfspaces.


\begin{theorem}\label[theorem]{thm:main-cube}
	For sufficiently large $n$ and $k$, with $k = O(n^{0.49})$, 
  define the concept class $\mathcal{C} = \{x \mapsto \cubek(Px): P \in \R^{k \times n}, PP^\T = I_k\}$. Any SQ algorithm that agnostically learns $\mathcal{C}$ to excess error \smash{$\frac12\paren{1 - \frac1{k^{0.49}}}$} requires $2^{n^{\Omega(1)}}$ queries or at least one query of tolerance \smash{$n^{-\Omega(\log (k) / \log\log k)}$}.
\end{theorem}


This strengthens the bounds of \cite{dkpz21} for the regime of $k = O(n^{0.49})$ and constant excess error, improving the $\smash{n^{-\tilde{\Omega}(\sqrt{\log k})}}$ tolerance requirement to $\smash{n^{-\tilde{\Omega}(\log k)}}$.
\Cref{thm:main-cube} follows directly from \Cref{lem:dkpz2.1,thm:dkpz-sq-lb} and the following lemma:


\begin{lem}\label[lemma]{lemma:cube-resilience}
For sufficiently large $k$, the function
	$\cubek$ is $\alpha$-approximately $d$-resilient for $\alpha ={ k^{-0.49}}$ and $d = \Omega(\log(k) / \log \log k)$.\end{lem}

The proof of \Cref{lemma:cube-resilience} has two main ingredients: a bound on the Hermite weight of $\cubek$ that is contained in its low-degree coefficients (\Cref{lemma:cube-low-deg-bound}), and an approximate resilience guarantee for functions with bounded low-degree Hermite weight (\Cref{lemma:resilience}).
We prove \Cref{lemma:cube-resilience} in \Cref{sssec:cube-resilience-proof} by applying those two lemmas and choosing an appropriate setting for $d$ in terms of $k$.

\begin{restatable}{lem}{lemmacubelowdegbound} \label[lemma]{lemma:cube-low-deg-bound}
For any sufficiently large $k$, and any $d \geq 0$,\footnote{See \Cref{ap:hermite} for notation for Hermite coefficients.}
  \begin{equation*}
    \sum_{\abs{J} \leq d} \widetilde{\cubek}(J)^2
    \leq \frac{20d(3\ln k)^d}{k} .   
  \end{equation*}
\end{restatable}


We prove this lemma in \Cref{sssec:cube-low-deg-coef} by exactly computing the Hermite coefficients of one-dimensional centered interval functions (\Cref{lemma:interval-hermite}) and using those values to carefully bound the $\cubek$ Hermite coefficients.
At a high level, our bounds on the low-degree Hermite coefficients of $\cubek$ are similar in flavor to the bounds of \citet{mansour92} on the low-degree Fourier coefficients of the read-once ``tribes'' CNF over the Boolean hypercube.

\begin{lem}\label[lemma]{lemma:resilience}
	For sufficiently large $k$, $d \geq 2$, and $f: \R^k \to \{-1, 1\}$, let $\gamma := \smash{\sum_{|J| \leq d} \widetilde{f}(J)^2}$.
	Then $f$ is $\alpha$-approximately $d$-resilient for 
	$\alpha = \smash{\gamma^{0.498} (72 \ln k)^{d/2}}.$	
\end{lem}

The proof of \Cref{lemma:resilience} is given in \Cref{sssec:infinite} and is somewhat technically involved.
Our argument modifies and extends a proof idea from \cite{dftww14}, which they use to show that the function $\mathsf{Tribes}: \flip^k \to \flip$ is approximately resilient.
Starting with the $\mathsf{Tribes}$ function, their approach is essentially to (i) discard its low-degree Fourier component; (ii) truncate the resulting function so it does not take very large values; (iii) again discard the low-degree Fourier component of (ii) (since the truncation could have reintroduced some low-degree Fourier component); and (iv) normalize the result of (iii) to give an $L^\infty$ norm of at most 1. They show that this yields a new function that (1) has zero low-degree Fourier weight, (2) takes output values that are bounded in $[-1,1]$, and (3) is close to the original $\mathsf{Tribes}$ function in $L^1$ distance.


In our setting we have the same high level goals of achieving (1-3), but achieving boundedness is significantly more difficult on the unbounded domain $\R^k$ than on the finite hypercube $\flip^k$.
Our witness to the approximate resilience of $\cubek$ is not constructed in a single shot (in contrast to \citeauthor{dftww14}), but rather is constructed gradually through an iterative process.


\subsection{Approximate resilience of functions with small low-degree weight (Proof of \Cref{lemma:resilience})}\label{sssec:infinite}


In this section, we prove \Cref{lemma:resilience}.
Our key tool is the $\dam{d,\tau}$ transformation, defined below, and a careful iterative application of $\dam{d,\tau}$ to produce a witness to the approximate resilience of a given Boolean function $f$ with small low-degree Hermite weight.





\begin{definition}
	For any $f \in L^2(\mnormalk)$ and $d \in \N$, let $\low{d}$ and $\high{d}$ be $L^2(\mnormalk) \to L^2(\mnormalk)$ transformations that reduce a function to its low-degree and high-degree Hermite components respectively, i.e.
	\begin{align*}
		\low{d}[f] := \sum_{\abs{J} \leq d} \tilde{f}(J) H_J, \quad \text{and} \quad
		\high{d}[f] := \sum_{\abs{J} > d} \tilde{f}(J) H_J = f - \low{d}[f].
	\end{align*}
	For any $\tau> 0$, the truncation transformation $\dam{d,\tau}: L^2(\mnormalk)  \to L^2(\mnormalk)  $ is
	\begin{align*}
	\dam{d,\tau}[f](x) &:= \high{d}[f](x) - \high{d}[f](x)\indicator{\abs{\low{d}[f](x)} > \tau} \\
	&\hphantom:= \begin{cases}
		\high{d}[f](x)  & \text{if $\abs{\low{d}[f](x)} \leq \tau$} , \\
		0 & \text{otherwise.}
		\end{cases}
	\end{align*}
\end{definition}

The purpose of $\dam{d,\tau}[f]$ is to shrink the low-degree weight of $f$ while staying bounded in $L^\infty$ and close to $f$ in $L^1$.
These properties are given in the following propositions.

\begin{prop}\label[proposition]{prop:small-tau}
	If $\norm[\infty]{f}<\infty$, then
	$\norml[\infty]{\dam{d,\tau}[f]} \leq \norm[\infty]{f} + \tau.$
\end{prop}
\begin{proof}
	$\dam{d,\tau}[f](x)$ is non-zero only if $\abs{\low{d}[f](x)} \leq \tau$.
	In that case, it is clear that $\abs{\dam{d,\tau}[f](x)} = \abs{\high{d}[f](x)} = \abs{f(x) - \low{d}[f](x)}\leq \abs{f(x)} + \tau$.
\end{proof}
\begin{restatable}{prop}{proplargetau} \label[proposition]{prop:large-tau}
	For any $k \geq 1$ and $d \geq 2$, fix some $a > 1$ and $\rho \geq \norm[2]{\low{d}[f]}$ and let 
\begin{equation}
\label{eq:taudef}
\tau := \rho \paren{4e \ln(3k) + \frac{8e}{d} \ln\paren{\frac{a \norm[2]{f}}{\rho}}}^{d/2}. 
\end{equation}
	Then, (i) $ \norml[2]{\low{d}[\dam{d,\tau}[f]]} \leq  \frac{\rho}{a}$, and (ii) $\norml[1]{\dam{d,\tau}[f] - f} \leq 2\rho$.
\end{restatable}
We prove \Cref{prop:large-tau} in \Cref{asec:proof-trunc-prop}.

Note that there is a tension in the choice of the truncation
parameter $\tau$. 
If $\tau$ is too large, then $\dam{d,\tau}[f]$ might still take large values.
But if $\tau$ is too small, then the low-degree weight of $\dam{d,\tau}[f]$ might
not become much smaller compared to that of $f$.
The proof of \Cref{lemma:resilience} works by applying $\dam{d,\tau}$
iteratively with a carefully chosen decreasing schedule of $\tau$-values.
This process converges to a function that is bounded, has zero low-degree
weight, and is sufficiently close to $f$, and this function certifies the
$\alpha$-approximate $d$-resilience of $f$.

	\medskip
	\begin{proof}[Proof of \Cref{lemma:resilience}]
Since any $f$ with $\norm[\infty]{f} \leq 1$ is trivially 1-approximately $d$-resilient for all $d \geq 0$, we may assume that $\alpha <1$. 
		We define a sequence of functions $(f_i)_{i\in \N}$ by $f_0 := f$ and $f_i := \dam{d,\tau_i}[f_{i-1}]$ for $i\geq1$, where
		\begin{equation}\label{eq:tau}
		\tau_i := \frac{\norm[2]{\low{d}[f_0]}}{4^{(i-1)d}} \paren{4e \ln(3k) + \frac{8e}{d} \ln\paren{\frac{4^{id} \norm[2]{f_{i-1}}}{\norm[2]{\low{d}[f_0]}}}}^{d/2}.
		\end{equation}
    We'll show that the sequence $(f_i)_{i \in \N}$ has a limit in $L^2(\mnormalk)$ that yields a witness to the $\alpha$-approximate $d$-resiliance of $f$.
    To do this, it will suffice to show the following claims for all $i \geq 1$:
		\begin{description}
		\item[\namedlabel{step:tau-bound}{Claim 1}.] $\tau_i \leq \frac\alpha{3\cdot 2^i}$.
		\item[\namedlabel{step:small-tau}{Claim 2}.] $\norm[\infty]{f_i} \leq 1 + \frac\alpha3 \sum_{\iota=1}^i \frac1{2^\iota} \leq 1 + \frac\alpha3$. 
		\item[\namedlabel{step:large-tau}{Claim 3}.] $\norm[2]{\low{d}[f_i]} \leq \frac{1}{ 4^{id}} \norm[2]{\low{d}[f_0]} \leq \frac\alpha{6 \cdot 4^{id}}$ and $\norm[1]{f_{i} - f_{i-1}} \leq \frac{\alpha}{3 \cdot 4^{(i-1)d}}$. 
		\end{description}
    We now explain why this is enough to prove the lemma.
    \ref{step:small-tau} ensures that $\norml[\infty]{f_i} \leq 1 + \alpha/3$, while \ref{step:large-tau} (for all $i$) ensures that $\norml[1]{f_i - f_0} \leq \sum_{\iota=1}^i \norml[1]{f_\iota - f_{\iota-1}} \leq 2\alpha/3$ (by the triangle inequality), and also $\lim_{i \to \infty} \norml[2]{\low{d}[f_i]} = 0$.
    By a limit argument (\Cref{prop:limit-argument}), the sequence $(f_i)_{i \in \N}$ converges in $L^2(\mnormalk)$ to some $f^* \in L^2(\mnormalk)$ with $\norml[\infty]{f^*} \leq 1 + \alpha/3$, $\low{d}[f^*] = 0$, and $\norml[1]{f^* - f} \leq 2\alpha/3$. 
    This proves the lemma because one of $f^*$ and $f^{**} := f^* / \norml[\infty]{f^*}$ witnesses that $f$ is $\alpha$-approximately $d$-resilient.
    Indeed, if $\norml[\infty]{f^*} > 1$, then
    $\norml[\infty]{f^{**}} = 1$,
    $\low{d}[f^{**}] = 0$, and
    \[
      \norm[1]{f - f^{**}}
      \leq 
      \norm[1]{f - f^*}
      + \norm[\infty]{f^* - f^{**}}
      \leq
      \frac{2\alpha}3
      + \paren{1 -  \frac{1}{\norm[\infty]{f^*}}}\norm[\infty]{f^*} 
      \leq \alpha ,
    \]
    where the first inequality uses the triangle inequality and comparison of $\norml[1]{\cdot}$ and $\norml[\infty]{\cdot}$.

It remains to prove \ref{step:tau-bound}, \ref{step:small-tau}, and \ref{step:large-tau} for all $i \geq 1$ by induction on $i$.
	
	For the base case $i = 1$, $\tau_1 \leq \frac\alpha6$ (\ref{step:tau-bound}) is an immediate consequence of the upper bound on $\tau_1$ from \Cref{fact:tau-bound} in \Cref{asec:proof-tau-bound}, which relies on having $\norml[\infty]{f_0} = 1 < \frac43$. 
	\Cref{prop:small-tau} and the bound on $\tau_1$ imply $\norm[\infty]{f_1} \leq 1 + \frac\alpha6$ (\ref{step:small-tau}).	
	By taking $a = 4^d$ and $\rho =\norm[2]{\low{d}[f_0]}$, \Cref{prop:large-tau} implies that $\norml[2]{\low{d}[f_1]} \leq \smash{\frac{1}{4^d}} \norm[2]{\low{d}[f_0]}$ and $\norml[1]{f_1 - f_0} \leq 2 \norm[2]{\low{d}[f_0]}$. 
	We conclude the base case of \ref{step:large-tau} by observing that $\norm[2]{\low{d}[f_0]} \leq \tau_1 \leq \frac\alpha6$ by \Cref{fact:tau-bound}. 
	
	We prove the inductive step by assuming that the three claims all hold for some fixed $i \geq 1$ and showing that they also hold for $i + 1$.
	By applying \Cref{fact:tau-bound} with $\norm[\infty]{f_i} \leq 1 + \frac\alpha3 \leq \frac43$ from step $i$ of \ref{step:small-tau}, we have $\tau_{i+1} \leq \frac{\alpha}{3 \cdot 2^{i+1}}$ (step $i+1$ of \ref{step:tau-bound}). 
	Step $i+1$ of \ref{step:small-tau} is immediate from \Cref{prop:small-tau}, the bound on $\tau_{i+1}$, and a geometric sum: \[\norml[\infty]{f_{i+1}} \leq \norml[\infty]{f_i} + \tau_{i+1} \leq 1 + \frac{\alpha}3 \sum_{\iota=1}^{i+1} \frac1{2^\iota} \leq 1 + \frac\alpha3.\]
	We apply \Cref{prop:large-tau} with $a = 4^d$ and $\rho = \smash{\frac{1}{4^{id}}} \norm[2]{\low{d}[f_0]}$\footnote{Note that $\rho \geq \norm[2]{\low{d}[f_{i}]}$ by step $i$ of \ref{step:large-tau}, which is necessary for \Cref{prop:large-tau}.} to obtain
  \[ \norml[2]{\low{d}[f_{i+1}]} \leq \smash{\frac{1}{ 4^{(i+1)d}}} \norm[2]{\low{d}[f_0]} \quad \text{and} \quad \norm[1]{f_{i+1} - f_{i}} \leq \smash{\frac{2}{4^{id}}} \norm[2]{\low{d}[f_0]} . \]
  Combining this with the bound $\norm[2]{\low{d}[f_0]}\leq \frac\alpha6$ completes step $i+1$ of \ref{step:large-tau}.
	
	Hence, the three claims hold for all $i \geq 1$ by induction, which concludes the proof.
\end{proof}

\subsection{Approximate resilience of $\cubek$ (Proof of \Cref{lemma:cube-resilience})}\label{sssec:cube-resilience-proof}
		Let $d = \frac{\ln k}{125 \ln\ln k}$.
	By \Cref{lemma:cube-low-deg-bound}, for sufficiently large $k$ we have
	{\begin{align*}
		\gamma := \sum_{\abs{J} \leq d} \widetilde{\cubek}(J)^2
    &\leq \frac{20d(3\ln k)^d}{k} .
	\end{align*}
	\Cref{lemma:resilience} guarantees that $\cubek$ is $\alpha$-approximately $d$-resilient for
	\begin{align*}
	\alpha 
	& = \gamma^{0.498} (72 \ln k)^{d/2}
	\leq \exp\paren{ -0.498 \ln k + 0.00799 \ln k + o(\ln k) }
	\leq k^{-0.49} .
	\end{align*}
	This completes the proof of \Cref{lemma:cube-resilience}.
	\hfill\qed



\newcommand{\boldf}{\boldsymbol{f}} 

\section{Hardness of SQ learning to constant excess error via weak learning lower bounds}\label{sec:lb-weak-learn}

In this section we give a different proof of our main agnostic SQ hardness result for learning intersections of $k$ halfspaces to constant excess error.  
While \Cref{thm:main-cube} established hardness for a highly structured subclass of this concept class (consisting of suitable embeddings of  the $\cubek$ function), the current argument only applies to the broader class of all intersections of $k$ halfspaces. However, an advantage of the current argument is that it holds for a wider range of values of $k$ (up to $2^{O(n^{0.245})}$). In more detail, in this section we prove the following:

\begin{theorem}\label[theorem]{thm:alt-intersection}
	For sufficiently large $n$ and any $k = 2^{O(n^{0.245})}$, any SQ algorithm that agnostically learns the class of intersections of $k$ halfspaces over $\R^n$ to excess error $c$ requires either $2^{n^{\Omega(1)}}$ queries or at least one query of tolerance $n^{-\Omega(\log (k) / \log\log k)}$. (Here $c>0$ is an absolute constant independent of all other parameters.)
\end{theorem}

As discussed in \Cref{sec:techniques}, the proof of \Cref{thm:alt-intersection} follows the high-level approach of Theorem~1.4 of   \cite{dftww14}. Rather than analyzing the Hermite spectrum of the hard-to-learn functions (as was done in \Cref{sec:lb-resilience}), the argument combines the agnostic learning algorithm of \cite{kkms} with a (slight extension of a) recently-established lower bound on the ability of membership query (MQ) algorithms to weakly learn intersections of halfspaces.

	
	\newcommand\dactual{\mathcal{D}_\mathrm{actual}}
	\newcommand\dist{\mathcal{D}}

We first recall the following lower bound from \cite{ds21}: 

\begin{lem} [\citealp{ds21}, Theorem~2] \label[lemma]{thm:ds21-lb}
For sufficiently large $m$, for any $q \geq m$, there is a distribution $\dactual$ over centrally symmetric convex sets of $\R^m$ with the following property:  for a target convex set $\boldf \sim \dactual$ for any MQ algorithm $A$ making at most $q$ many queries to $\boldf$, the expected error of $A$ (the probability over $\boldf \sim \dactual$, any internal randomness of $A$, and a Gaussian $\bx \sim \mnormaln$, that the output hypothesis $h$ of $A$ is wrong on $\bx$) is at least $\smash{\frac12 - {\frac {O(\log q )}{\sqrt{m}}}}$.
\end{lem}

We require the following corollary of \Cref{thm:ds21-lb}, which we prove in \Cref{asec:weak-learn-cor}.

\begin{restatable}{cor}{corollaryds} \label[corollary]{cor:ds}
	For sufficiently large $n$, for all $q \geq m$, there is a distribution $\mathcal{D}$ over intersections of $q^{101}$ halfspaces such that for a target function $\boldf \sim {\cal D}$, any MQ algorithm $\mathcal{A}$ making at most $q$ queries to $\boldf$ has expected error at least $\half - \frac{O(\log q)}{\sqrt{m}}$ (where the expectation is over  $\boldf \sim \dist$ and any internal randomness of $\mathcal{A}$, and the the accuracy is with respect to $\mnormaln$).
\end{restatable}

\Cref{thm:alt-intersection} follows immediately from \Cref{thm:dkpz-sq-lb} and the following lemma.

\begin{lem}\label[lemma]{lemma:random-intersect}
	For any $k = 2^{O(n^{0.245})}$, there exists an intersection of $k$ halfspaces $f: \R^m \to \bits$ that has $L^1$  $\half$-approximate degree $d=\Omega(\log(k) / \log\log k)$, where $m = O(n^{0.49}).$
\end{lem}
	
\begin{proof}
\newcommand{\chalf}{\mathcal{H}_{q^{101}}}
First we note that we may assume $k$ is at least some sufficiently large absolute constant as specified below through the choice of $q$ (since otherwise, because of the $\Omega(\cdot)$ in the specification of $d$, there is nothing to prove). Suppose that every intersection of $k$ halfspaces $f$ over $\R^m$ has $L^1$ $\half$-approximate degree at most $d-1$; we will prove the lemma by showing that $d$ must be $\Omega(\log(k)/\log \log k).$
	
	Let $q=k^{1/101}$ and
	let $\mathcal{S} \subseteq \R^n$ be the subspace of $\R^n$ spanned by the first $m = c_1 \ln^2 q$ coordinates, where $c_1$ is a sufficiently large universal constant specified below and $q$ is chosen sufficiently large (relative to $c_1$) so that $q \geq m$ and $m$ satisfies the ``sufficiently large'' requirement of \Cref{cor:ds}.
	By \Cref{cor:ds}, there is a distribution $\dist$ over intersections of at most $k=q^{101}$ halfspaces over $\mathcal{S}$ such that any membership query algorithm making at most $q$ queries to an unknown $\boldf \sim {\cal D}$ outputs a hypothesis with expected error at least $\smash{\half - \frac{O(\log q)}{\sqrt{m}}}$.
	For a sufficiently large setting of $c_1$, this expected error is at least $\smash{\half - \frac{O(\log q)}{\sqrt{c_1} \ln q}} \geq 0.49.$

	By the assumption that every intersection of $k$ halfspaces has $L^1$ $\half$-approximate degree at most $d-1$, if the\ignore{SQ version of the} agnostic learner of Theorem~5 of \cite{kkms} is run on any intersection of $k$ halfspaces over the first $m$ coordinates, then it uses $s:= \poly(m^d / \epsilon)$ labeled examples from $\mnormaln$, runs in $\poly(s)$ time, and with probability at least (say) $0.9$ outputs a hypothesis $h$ with error at most 
	\[\eps + \frac{1}{2} \min_{p \in \calP_{n,d}} \norm[1]{f-p} \leq \eps + \frac14\]
(see Theorem~1.3 of \cite{dftww14}).
Taking $\eps = 0.15$ and observing that a labeled example from $\mnormaln$ can be simulated using a single membership query, we see that for the concept class of intersections of $k$ halfspaces over the first $m$ coordinates, there is a membership query algorithm ${\cal A}$ that makes at most $m^{c_2d}$ many membership queries and with probability at least $0.9$ achieves error at most $0.4$; hence the expected error of this MQ algorithm is at most $0.9 \cdot 0.4 + 0.1 \cdot 1 = 0.46$.
	
	Comparing the conclusions of the previous two paragraphs, we see that $m^{c_2 d} \geq q,$ and hence (recalling that $m= c_1 \ln^2 q$ and $q=k^{1/101}$), we get that
\[
d \geq {\frac {\ln q}{c_2 \ln m}}=\Omega(\log(k)/\log \log k),
\]
which proves the lemma.		
\end{proof}

\newcommand{\step}{\mathsf{step}}

\section{Hardness of SQ learning to arbitrary excess error} \label{sec:ganzburg}


In this section, we strengthen both of the SQ lower bounds from \Cref{sec:lb-resilience,sec:lb-weak-learn} by combining them with lower bounds on the $L^1$ $\eps$-approximate degree of halfspaces due to \cite{ganzburg02}.
By doing so, we improve the lower bounds to $n^{\Omega(\log (k)/\log \log (k) + 1/\eps^2)}$ for agnostically learning intersections of $k$ halfspaces to excess error $\eps$ for any {$\smash{\eps\geq n^{-c}}$ (cf.~\Cref{thm:dkpz-sq-lb})}.
By the arithmetic-geometric mean inequality, this lower bound is always at least as strong as the $n^{\smash{-\tilde\Omega(\log^{1/2} (k)/\eps)}}$ lower bound of \cite{dkpz21}.  




{Let $k,m$ be as described in \Cref{lemma:random-intersect}.} 
We construct an intersection of $k+1$ halfspaces over $\R^{m+1}$ by taking the intersection of
\begin{itemize}
  \item the $k$ halfspaces identified in \Cref{lemma:random-intersect} over $\R^m$; and 
   
  \item an origin-centered halfspace orthogonal to the {$(m+1)$}-st coordinate basis vector.
\end{itemize}

\Cref{asec:ganzburg-proofs} formally bounds the $L^1$ approximate degree of intersections of this construction of half spaces and proves a strengthening of \Cref{thm:alt-intersection}, which is our main agnostic SQ lower bound:

\begin{restatable}[Formal version of \Cref{thm:informal}]{theorem}{theoremaltcubeganzburg} \label[theorem]{thm:alt-cube-ganzburg}
  For any $k = 2^{O(n^{0.245})}$ and any $\eps \geq n^{-c}$ for a suitably small absolute constant $c>0$, any SQ algorithm that agnostically learns intersections of $k$ halfspaces to excess error $\epsilon$ under Gaussian marginals requires either $2^{n^{\Omega(1)}}$ queries or at least one query of tolerance $n^{-\Omega(\log (k) / \log\log k  + 1/\eps^2)}$.
\end{restatable}

\bibliography{ourbibliography,bibliography}
\clearpage

\appendix

\hyphenation{non-agnostic}

\section{Discussion}\label{sec:discussion}

\newcommand{\GNS}{\mathrm{GNS}}

\subsection{Agnostic SQ lower bounds for learning  functions of bounded Gaussian surface area and convex $m$-subspace juntas} \label{sec:gsa-subspace-junta}

In this appendix, we note that our arguments imply agnostic SQ lower bounds for several classes of $\bits$-valued functions over $\R^n$ that were studied by \citet{vempala10} and \citet{kos08}. Our lower bounds essentially match the upper bounds for those classes by \citet{kos08}.

\medskip

\noindent \textbf{Functions with bounded Gaussian Surface Area.}  Recall the definition of Gaussian Surface Area:
\begin{definition}
	Let $f: \R^n \to \bits$ be such that $\{x \in \R^n: f(x) = 1\}$ is a Borel set. The \emph{Gaussian surface area} of $f$ is defined to be
	\[\Gamma(f) := \liminf_{\delta \to 0} \frac{\pr[\bx \sim \mnormaln]{f(\bx) = -1 \ \text{and} \ \exists y \in f^{-1}(1) \ \text{s.t.} \ \norm[2]{\bx - y} \leq \delta}}{\delta}. \]
\end{definition}

Let ${\cal C}_s$ denote the class of all Borel sets in $\R^n$ with Gaussian surface area at most $s$.
The main result of \citet[Theorem~25]{kos08} is that ${\cal C}_s$ is agnostically learnable to accuracy $\OPT + \eps$ by an SQ algorithm that makes $n^{O(s^2/\eps^4)}$ queries, each of tolerance $n^{-O(s^2/\eps^4)}$.
Their
agnostic learning algorithm for intersections of $k$ halfspaces, mentioned earlier, is obtained from this result by combining it with the fact, due to \citet{nazarov03}, that any intersection of $k$ halfspaces has Gaussian surface area at most $O(\sqrt{\log k}).$

Let $s = O(n^{0.1225})$, let $m=n^{0.49}$ and let $k=\smash{2^{s^2} = 2^{O(\sqrt{m})}}$.
By \Cref{lemma:random-intersect} there is an intersection of $k$ halfspaces over $\R^m$ that has $L^1$ $\half$-approximate degree $\Omega(s^2 / \log s)$, and by \citeauthor{nazarov03}'s upper bound on Gaussian surface area, this function has Gaussian surface area at most $O(s)$. Combining this with \Cref{thm:dkpz-sq-lb}, we immediately obtain that for any $s \leq O(n^{0.1225})$, any SQ agnostic learning algorithm that achieves constant excess error under Gaussian marginals for the class ${\cal C}_s$ either requires queries with tolerance at most $n^{-\Omega(s^2/\log(s))}$ or makes at least $2^{n^{\Omega(1)}}$ queries.  Combining this with the arguments of \Cref{sec:ganzburg}, we get the following result for ${\cal C}_s$:

\begin{theorem} \label[theorem]{thm:GSAlb}
	For sufficiently large $n$, any $s = O(n^{0.1225})$, and any $\eps \geq n^{-c}$ for a suitably small absolute constant $c>0$, any SQ algorithm that agnostically learns the class ${\cal C}_s$ to excess error $\epsilon$ requires either $2^{n^{\Omega(1)}}$ queries or at least one query of tolerance $n^{-\Omega(s^2 / \log(s)  + 1/\eps^2)}$.
\end{theorem}

\medskip

\noindent \textbf{Convex subspace juntas.} \citet{vempala10} gave a learning algorithm (in the realizable, i.e.,~non-agnostic, setting) for a class of functions that we refer to as \emph{convex $m$-subspace juntas}. A function $f: \R^n \to \bits$ is a convex $m$-subspace junta if $f$ is the indicator function of a convex set $K$ with a normal subspace of dimension $m$; equivalently, $f$ is an intersection of halfspaces all of whose normal vectors lie in some subspace of $\R^n$ of dimension at most $m$ (note that the number of halfspaces in such an intersection may be arbitrarily large or even infinite).

Vempala's algorithm learns to accuracy $\eps$ and runs in time $\poly(n,2^m/\eps,m^{\tilde{O}(\sqrt{m}/\eps^4)})$ in the realizable ($\OPT=0$) setting of learning under Gaussian marginals.  
As alluded to in \Cref{sec:intro}, this algorithm 
uses principal component analysis to do a preprocessing step and then runs the algorithm of \cite{kos08}.
The analysis crucially relies on a Brascamp-Lieb type inequality \citep[Lemma~4.7 of][]{vempala10} which, roughly speaking, makes it possible to identify the ``relevant directions''); however, this breaks down in the non-realizable (agnostic) setting. The best known agnostic learning result for the class of convex $m$-subspace juntas under Gaussian marginals is the {SQ algorithm of \citet{kos08}, which makes $n^{O(\sqrt{m}/\eps^4)}$ statistical queries, each of tolerance at least $n^{-O(\sqrt{m}/\eps^4)}$. This performance bound for the algorithm} follows immediately from Theorem~25 of \citet{kos08} and the upper bound, due to \citet{Ball:93}, that any convex set in $\R^m$ has Gaussian surface area at most $O(m^{1/4}).$

Let $m \leq  n^{0.49}.$ By \Cref{lemma:random-intersect} there is a convex $m$-subspace junta (an intersection of $2^{O(\sqrt{m})}$ many halfspaces, all of whose normal vectors lie in an $m$-dimensional subspace of $\R^n$) that has $L^1$ $\half$-approximate degree $\Omega(\sqrt{m}/\log m).$ Combining this with  \Cref{thm:dkpz-sq-lb} and the arguments of \Cref{sec:ganzburg}, we obtain the following lower bound:

\begin{theorem} \label[theorem]{thm:junta-lb}
	For sufficiently large $n$, any $m \leq n^{0.49}$, and $\eps \geq n^{-c}$ for a suitably small absolute constant $c>0$, any SQ algorithm that agnostically learns the class of convex $m$-subspace juntas to excess error $\epsilon$ requires either $2^{n^{\Omega(1)}}$ queries or at least one query of tolerance $n^{-\Omega(\sqrt{m}/\log m  + 1/\eps^2)}$.
\end{theorem}

\subsection{On lower bounds for $L^1$ polynomial approximation}

One of the contributions of \cite{dkpz21} is that it introduced new analytic techniques for obtaining lower bounds on the $L^1$ approximate degree of functions $f: \R^n  \to \bits$.
In particular, \cite{dkpz21} established a new structural result that translates a lower bound on the Gaussian Noise Sensitivity of any function $f: \R^n \to \bits$ to a lower bound on the $L^1$ approximate degree of $f$.  

\begin{definition} [\citealp{odonnell14}, Definition~11.9] \label[definition]{def:GNS}
Given $0\leq \rho  \leq 1$ and $f: \R^n  \to \bits$, the \emph{Gaussian Noise Sensitivity of $f$ at correlation $1-\rho$}, written $\GNS_\rho(f)$, is 
\[
\GNS_\rho(f) := \Pr_{(\bx,\bg) \sim {\cal N}(0,I_n)^{\otimes2}}\Bigl[f(\bx) \neq f((1-\rho)\bx + \sqrt{2 \rho - \rho^2} \bg)\Bigr] .
\]
Equivalently, $\GNS_\rho(f)$ is the probability that $f(\bx)\neq f(\by)$ where $\bx,\by$ are standard $n$-dimensional Gaussians with correlation $1-\rho$.
\end{definition}

\begin{theorem} [\citealp{dkpz21}, Theorem~1.5] \label[theorem]{thm:gns}
Let $f: \R^n  \to \bits$ and let $p: \R^n \to \R$ be any polynomial of degree at most $d$. Then

\begin{enumerate}

\item $\norm[1]{f - p} \geq \Omega(1/\log d) \cdot \GNS_{(\ln(d)/d)^2}(f).$

\item For any $\eps > 0$, we have $\norm[1]{f - p}  \geq \GNS_\eps(f)/4 - O(d \sqrt{\eps}).$

\end{enumerate}
\end{theorem}

In contrast with $L^2$ polynomial approximation (for which the degree required for $\eps$-approximation can be ``read off'' of the Hermite expansion), polynomial approximation in $L^1$ is much less well understood.  Thus it is interesting and useful to have general tools for $L^1$ approximate degree bounds such as \Cref{thm:gns}, and conversely, it is of interest to understand the limitations of such tools.

\citet{dkpz21} use \Cref{thm:gns} to prove an $L^1$ approximate degree lower bound for intersections of $k$ halfspaces.  They first show that for a particular\footnote{This function $f'$ is very similar to the $\cubek$ function; instead of upper and lower bounding each of the $k$ coordinates $x_1,\dots,x_k$, it only upper bounds each coordinate.} intersection of $k$ halfspaces $f'$ over $\R^k$, for each $\tau < \Theta(1/\log k)$ it holds that $\GNS_\tau(f') = \Theta(\sqrt{\tau \log k})$. Combining this with item (1) of \Cref{thm:gns} gives that any polynomial $p$ for which $\|f-p\|_1 \leq \eps$ must have $d \geq \Omega(\smash{\frac {\log^{1/2} k}{\eps}}).$
Our resilience results for the $\cubek$ function give a stronger $L^1$ approximate degree lower bound, and combining this with the $\GNS$ bound from \citet{dkpz21} gives an example of a function for which the bound of part (1) of \Cref{thm:gns} is not tight. 

In more detail, recall that our \Cref{lemma:cube-resilience} states that the $\cubek$ function is $k^{-0.49}$-approximately $\Theta(\log(k)/\log \log k)$-resilient. An entirely similar analysis to the proof of \Cref{lemma:cube-resilience} shows that the function $f'$ of \citet{dkpz21} is also $k^{-0.49}$-approximately $d:=\Theta(\log(k)/\log \log k)$-resilient, i.e., there is a function $g: \R^k \to [-1,1]$ which has zero correlation with every polynomial of degree at most $d-1$ and which has $\|f'-g\|_1 \leq k^{-0.49}.$ 
By \Cref{lem:dkpz2.1},
the existence of this resilient $g$ implies that every polynomial $p$ of degree at most $d-1$ must have
\begin{equation*}
  \|f'-p\|_1 \geq 1 - \frac2{k^{0.49}} ,
\end{equation*}
which is close to one for large $k$.

Now consider what can be obtained from the $\GNS$ bound of \citet{dkpz21}.
Since 
\[
\GNS_{((\ln (d-1))/(d-1))^2}(f') = 
\sqrt{{\frac {\Theta((\log \log k)^4)}{(\ln k)^2}} \cdot \ln k} = {\frac {\Theta((\log \log k)^2)}{\sqrt{\ln k}}},
\]
part (1) of \Cref{thm:gns} only gives that every polynomial $p$ of degree at most $d-1$ has
\begin{equation*}
  \|f'-p\|_1 \geq \Omega\left(\frac{\log\log(k)}{\sqrt{\log k}} \right) .
\end{equation*}
This bound is close to zero for large $k$.


\section{Hermite polynomials and Gaussian hypercontractivity} \label{asec:hermite}
\label{ap:hermite}

Let $\{ h_j \}_{j=0}^\infty$ be the (unnormalized) probabilists' Hermite polynomials
\begin{equation} \label{eq:hermite-defn}
  h_j(x) := (-1)^j e^{x^2/2} \frac{\dif{^j}}{\dif x^j} e^{-x^2/2} , \quad j = 0, 1, 2, \dotsc .
\end{equation}
These polynomials form an orthogonal basis for the Hilbert space $L^2(\normal)$; more precisely, we have $\ip{h_j, h_{j'}} = j! \cdot \delta_{j, j'}$.
For any $f \in L^2(\normal)$, the Hermite coefficients $\widetilde{f}(j)$ of $f$ are given by \[\widetilde{f}(j) := \frac{1}{\sqrt{j!}} \innerprod{f}{h_j}.\]

Let $\{H_J\}_{J \in \N^k}$ be the multivariate Hermite polynomials, which correspond to a tensor product of the univariate Hermite polynomials above.
That is,
\[H_J(x) := \prod_{i=1}^k h_{J_i}(x_i).\]
These polynomials form an {orthogonal} basis for $L^2(\mnormalk)$, and we have that $\innerprod{H_J}{H_{J'}} = J! \delta_{J, J'}$, where {$J! = J_1! \cdots J_k!$.}
For any $F \in L^2(\mnormalk)$, the Hermite coefficients $\widetilde{F}(J)$ of $F$ are given by \[\widetilde{F}(J) := \frac{1}{\sqrt{J!}} \innerprod{F}{H_J}.\]
Additional properties of the Hermite polynomials can be found in
Chapter 22 of \citet{AbramowitzStegun:72} and Section~11.2 of \citet{odonnell14}.

Our results---particularly those in \Cref{sssec:infinite}---rely on bounds powered by Gaussian hypercontractivity.
We recall the basic Gaussian hypercontractive inequality for low-degree polynomials 
\citep{Bon70,Nelson73,Gross75}:

\begin{fact}\label[fact]{fact:hypercontractivity}
{For  a polynomial $p \in \mathcal{P}_d$ and any $q \geq 2$, $\norm[q]{p} \leq (q-1)^{d/2} \norm[2]{p}$.}
\end{fact}

In particular, we will use the following bound on the fourth moment of Hermite polynomials, which follows immediately from \Cref{fact:hypercontractivity} and standard bounds on the norm of Hermite polynomials:

\begin{fact}\label[fact]{fact:hermite4}
	$\norm[4]{H_J} \leq 3^{d/2} \norm[2]{H_J} \leq 3^{d/2} \sqrt{J!}$.
\end{fact}

We will also use the following concentration bound, which follows from Gaussian hypercontractivity using Markov's inequality:\ignore{ (see \Cref{ap:poly-bound} for the simple proof):}

\begin{fact} [\citealp{odonnell14}, Theorem~9.23] \label[fact]{fact:poly-bound}
	For any polynomial $p: \R^k \to \R$ of degree $d$ and any $t \geq e^d$,
	\[\Pr_{\bx \sim \mnormaln}{[\abs{p(\bx)} \geq t \norm[2]{p}]} \leq \exp\paren{-\frac{d}{2e} t^{2/d}}.\]
\end{fact}



\begin{proof}
Consider any $q \geq 2$.
\begin{align*}
	\pr{\abs{p(\bx)} \geq t \norm[2]{p}}
	&= \pr{\abs{p(\bx)}^q \geq t^q \norm[2]{p}^q}
	\leq \frac{\norm[q]{p}^q}{t^q \norm[2]{p}^q}
	\leq \paren{\frac{(q - 1)^{d/2}}{t}}^{q} 
	\leq \paren{\frac{q^{d/2}}{t}}^{q} .
\end{align*}

Let $q = \frac{t^{2/d}}{e}$, which has $q \geq 2$ because $t \geq e^d$. Then,
\begin{align*}
	\pr{\abs{p(\bx)} \geq t \norm[2]{p}}
	&\leq \paren{\frac{1}{e^{d/2}}}^{t^{2/d} / e}
	= \exp\paren{- \frac{d}{2e} t^{2/d}}. \qedhere
\end{align*}

\end{proof}


\section{Supporting lemmas and proofs for \Cref{sec:lb-resilience}}\label{asec:intervals}

\subsection{Small low-degree Hermite weight of $\cubek$ (Proof of \Cref{lemma:cube-low-deg-bound})}\label{sssec:cube-low-deg-coef}

We recall \Cref{lemma:cube-low-deg-bound}.

\lemmacubelowdegbound*

We note that by the analysis of $\cubek$ by \citet[Example~14]{de2021convex}, the upper bound of \Cref{lemma:cube-low-deg-bound} in the case $d=2$ is tight up to constant factors.

%

Our proof of \Cref{lemma:cube-low-deg-bound} uses the product structure of $\mnormalk$  and the fact that $\cubek$ is essentially a product of univariate interval functions over disjoint variables. Thanks to these properties, it suffices to analyze the Hermite coefficients of interval functions of the right width.

For any $\theta \geq 0$, let $f_\theta:\R \to \bit$ be the indicator function for the interval $[-\theta, \theta]$, i.e.,
\begin{equation*}
  f_\theta(x) := \indicator{\abs{x} \leq \theta} .
\end{equation*}
Then, $\cubek$ can be written as
\begin{equation*}
  \cubek(x) = 2 \prod_{i=1}^k f_{\theta_k}(x_i) - 1.
\end{equation*}
Since $\theta_k$ is chosen to ensure that $\EE[\bx \sim \mnormalk]{\cubek(\bx)} = 0$, the Hermite coefficients of $\cubek$ are given by
\begin{equation*}
  \widetilde{\cubek}(J)
  =
  \begin{cases}
    0 & \text{if $J = 0$} , \\
    2 \prod_{i=1}^k \widetilde{f_{\theta_k}}(J_i) & \text{otherwise} .
  \end{cases}
\end{equation*}

\begin{proof}[Proof of \Cref{lemma:cube-low-deg-bound}]
  We may assume that $d \leq k / (2e^2\ln k)$, since otherwise the claimed bound on $\sum_{|J| \leq d} \widetilde{\cubek(J)}^2$ is more than one.

  By \Cref{lemma:interval-hermite} (stated and proved below), $\widetilde{f_{\theta_k}}(J_i) = 0$ for any odd $J_i$.
  Hence, the only Hermite coefficients that may be non-zero are those corresponding to multi-indices $J \in \N^k$ with (i) only even components, and (ii) $1 \leq |J| \leq d$.
  Let $\mathcal{J}$ denote this set of multi-indices.
  For any such $J \in \mathcal{J}$,
  \begin{align*}
    \widetilde{\cubek}(J)^2
    = 4\prod_{i=1}^k \widetilde{f_{\theta_k}}(J_i)^2
    & = 4\prod_{i : J_i = 0} \widetilde{f_{\theta_k}}(J_i)^2
    \prod_{i : J_i \geq 2} \widetilde{f_{\theta_k}}(J_i)^2 \\
    & \leq 4 \prod_{i : J_i \geq 2} \left[ \left( 1 + \sqrt{\frac{e}{J_i}} \theta_k \right)^{2(J_i-1)} e^{-\theta_k^2} \right] \\
    & \leq 4 \left( 1 + \sqrt{\frac{e}{2}} \theta_k \right)^{2(|J|-\#J)} e^{-\theta_k^2 \#J} ,
  \end{align*}
  where the first inequality uses the fact that $|\widetilde{f_{\theta_k}}(0)| \leq 1$ (\Cref{lemma:interval-hermite}) and the bound from \Cref{lemma:complex-hermite-bound} (stated and proved below).
  To bound the sum $\sum_{|J| \leq d} \widetilde{\cubek}(J)^2 = \sum_{J \in \mathcal{J}} \widetilde{\cubek}(J)^2$, we partition the terms by the value of $\#J$.
  Note that $\#J$ must satisfy $1 \leq \#J \leq \floor{d/2}$, since $J$ is not all zeros, and every non-zero component of $J$ is at least two.
  Therefore,
  \begin{align}
    \sum_{|J| \leq d} \widetilde{\cubek}(J)^2
    & = \sum_{t=1}^{\floor{d/2}} \sum_{J \in \mathcal{J} : \#J=t} \widetilde{\cubek}(J)^2  \nonumber \\
    & \leq
    4 \left( 1 + \sqrt{\frac{e}{2}} \theta_k \right)^{2(d-1)} 
    \sum_{t=1}^{\floor{d/2}} |\{ J \in \mathcal{J} : \#J=t \}| \cdot
    e^{-\theta_k^2 t} .
    \label{eq:hermite}
  \end{align}
The definition of $\mathcal{J}$ and standard binomial coefficient inequalities provide a bound on the number of $J \in \mathcal{J}$ with $\#J=t$ for $t \geq 1$:
  \begin{align*}
  	\abs{\{J \in \mathcal{J}: \# J = t\}} 
	&= {k \choose t} \abs{\{S \in \N^t: \abs{S} \leq \floor{d/2}, S_i > 0 \text{~for all $i \in [t]$}\}} \\
	&= {k \choose t} \abs{\{S \in \N^t: \abs{S} \leq \floor{d/2} - t\}} \\
	&= {k \choose t} {\floor{d / 2} \choose t} 
	\leq \paren{\frac{e^2 kd}{2t^2}}^t.
  \end{align*}
  Therefore, we can bound the final expression from \eqref{eq:hermite} by
  \begin{align*}
    4 \left( 1 + \sqrt{\frac{e}{2}} \theta_k \right)^{2(d-1)} 
    \sum_{t=1}^{\floor{d/2}} \left( \frac{e^2kd}{2t^2} e^{-\theta_k^2} \right)^t
    & \leq 4 \left( 1 + \sqrt{e\ln k} \right)^{2(d-1)} 
    \sum_{t=1}^{\floor{d/2}} \left( \frac{e^2d\ln k}{t^2k} \right)^t
    \\
    & \leq 8 \left( 1 + \sqrt{e\ln k} \right)^{2(d-1)} \cdot \frac{e^2d\ln k}{k}
    \leq \frac{20d(3\ln k)^d}{k} ,
  \end{align*}
  where the first inequality uses the bounds on $\theta_k$ from \Cref{lemma:mills},
  and the second inequality uses
  the assumption $d \leq k / (2e^2\ln k)$.
\end{proof}

The preceding proof relies on three supporting lemmas:
\Cref{lemma:interval-hermite} and \Cref{lemma:complex-hermite-bound} compute and bound the Hermite coefficients of $f_\theta$;
\Cref{lemma:mills} gives upper- and lower-bounds on $\theta_k$.

Let $\phi(x) = \frac1{\sqrt{2\pi}} e^{-x^2 / 2}$ denote the probability density function of the one-dimensional Gaussian distribution $\normal$.

\begin{lem} \label[lemma]{lemma:interval-hermite}
For all $j \geq 0$, the Hermite coefficients of $f_\theta$ are as follows:
  \begin{equation*}
    \widetilde{f_\theta}(j)
    =
    \begin{cases}
      \int_{-\theta}^\theta \phi(x) \dif x & \text{if $j=0$} , \\
      0 & \text{if $j$ is odd} , \\
      -\frac2{\sqrt{j!}} h_{j-1}(\theta) \phi(\theta) & \text{if $j\geq2$ is even} .
    \end{cases}
  \end{equation*}
\end{lem}

\begin{proof}
Recalling the definition of univariate Hermite polynomials from \Cref{asec:hermite}, the degree-$0$ coefficient $\widetilde{ f_\theta}(0)$ is
\begin{equation*}
  \widetilde{ f_\theta}(0)
  = \int_{-\infty}^\infty f_\theta(x) \phi(x) \dif x
  = \int_{-\theta}^\theta \phi(x) \dif x .
\end{equation*}
The degree-$j$ coefficient, for $j \geq 1$, is
\begin{align*}
  \widetilde{ f_\theta}(j)  & = \frac1{\sqrt{j!}} \int_{-\infty}^\infty f_\theta(x) h_j(x) \phi(x) \dif x
  = \frac1{\sqrt{j!}} \int_{-\theta}^\theta h_j(x) \phi(x) \dif x\\
  & = \frac1{\sqrt{j!}} \int_{-\theta}^\theta -\frac{\dif}{\dif x} \bracket{ h_{j-1} \phi(x) } \phi(x) \dif x  = \frac1{\sqrt{j!}} \eval[\Big]{-h_{j-1}(x) \phi(x)}_{-\theta}^\theta \\
  & = \frac1{\sqrt{j!}}
  \paren{ h_{j-1}(-\theta) - h_{j-1}(\theta) } \phi(\theta) .
\end{align*}
The third equality follows from the identity
$
  h_j(x) \phi(x) = -\frac{\dif}{\dif x} \bracket{ h_{j-1} \phi(x) }
$
for $j\geq1$, which follows from the definition in \eqref{eq:hermite-defn}.
The last equality uses that $\phi(x)$ is an even function. 

Furthermore, if $j$ is odd, then $h_{j-1}(-\theta) = h_{j-1}(\theta)$, and hence $\widetilde{ f_\theta}(j) = 0$.
If $j$ is even and $j \geq 2$, then $h_{j-1}(-\theta) = -h_{j-1}(\theta)$, and hence $\widetilde{ f_\theta}(j) = -\frac2{\sqrt{j!}} h_{j-1}(\theta) \phi(\theta)$.
\end{proof}

We bound the even-degree Hermite coefficients of the interval function by bounding each univariate Hermite polynomial, which provides the following coefficient bound.

\begin{lem} \label[lemma]{lemma:complex-hermite-bound}
  For any even $j\geq2$ and any $\theta \geq 0$, 
  \begin{equation*}
    \widetilde{ f_\theta}(j)^2 = \frac4{j!} h_{j-1}(\theta)^2 \phi(\theta)^2 \leq
      \paren{ 1 + \theta \sqrt{\frac{e}{j}} }^{2(j-1)}
      e^{-\theta^2} .
  \end{equation*}
\end{lem}

\begin{proof}
The equality is by \Cref{lemma:interval-hermite}. For the inequality, we define the following values:
  \[
    A_{j,\theta}:= \frac1{\sqrt{j!}} \abs{h_{j-1}(\theta)}
    ,
    \qquad
    B_{j,\theta}:= 
      \sqrt[4]{\frac{2e^2}{\pi j^3}}
      \paren{ 1 + \theta \sqrt{\frac{e}{j}} }^{j-1}
    .
  \]
  We show that $A_{j,\theta}\leq B_{j,\theta}$. Since $\widetilde{ f_\theta}(j)^2 = (2/\pi)A_{j,\theta}^2 \cdot e^{-\theta^2}$, this inequality implies that $\widetilde{f_\theta}(j)^2$ is at most $(2/\pi)B_{j,\theta}^2 \cdot e^{-\theta^2}$, which is easily verified to be at most the claimed upper bound in the statement of \Cref{lemma:complex-hermite-bound}.

We expand $A_{j,\theta}$ using an explicit formula for the Hermite polynomial~\citep[Equation 22.3.11]{AbramowitzStegun:72}, followed by a change of variable:
	\begin{align*}
    A_{j,\theta}& = \frac1{\sqrt{j!}} \abs{h_{j-1}(\theta)} \\
                & = \frac{(j-1)!}{\sqrt{j!}} \abs{ \sum_{m=0}^{j/2-1} \frac{(-1)^m \theta^{j-1-2m}}{2^m m! (j-1-2m)!} } && \text{(explicit formula for $h_{j-1}(\theta)$)} \\
                & = \frac{(j-1)!}{\sqrt{j!}} \abs{ \sum_{\text{odd} \ \ell=1}^{j-1} \frac{(-1)^{\frac{j-1-\ell}{2}} \theta^\ell}{2^{\frac{j-1-\ell}{2}} \paren{\frac{j-1-\ell}{2}}! \ell!} } . && \text{(change of variable)}
    \end{align*}
Thus, by the triangle inequality,
	\begin{align}
    A_{j,\theta}
    %
      & \leq \sum_{\text{odd} \ \ell=1}^{j-1} \frac{(j-1)!}{\sqrt{j!}} \cdot \frac{\theta^\ell}{2^{\frac{j-1-\ell}{2}} \paren{\frac{j-1-\ell}{2}}! \ell!}
    = \sum_{\text{odd} \ \ell=1}^{j-1} \frac{\sqrt{2}}{2^{j/2}} \binom{j-1}{\ell} \frac{(j-1-\ell)!}{\sqrt{j!} \paren{\frac{j-1-\ell}{2}}!} (\sqrt2\theta)^\ell .
    \label{eq:Ajtheta-bound}
    \end{align}
    We employ Stirling's approximation $\sqrt{2\pi n}(n/e)^n e^{1/(12n+1)} \leq n! \leq \sqrt{2\pi n} (n/e)^n e^{1/(12n)}$ to bound each term in the sum from \eqref{eq:Ajtheta-bound}. For any odd $\ell \in [1, j - 3]$:
        \begin{align*}
      \frac{\sqrt{2}}{2^{j/2}}  \binom{j-1}\ell \frac{(j-1-\ell)!}{\sqrt{j!} \paren{\frac{j-1-\ell}{2}}!} (\sqrt2\theta)^\ell
      & \leq \frac{\sqrt2}{2^{j/2}} \binom{j-1}{\ell}
      \sqrt[4]{\frac{2}{\pi j}}
      \paren{ \sqrt{\frac{e}{j}} }^j
      \paren{ \frac{2(j-1-\ell)}{e} }^{\frac{j-1-\ell}2}
      (\sqrt2\theta)^\ell \\
      & = \binom{j-1}\ell
      \sqrt[4]{\frac{2e^2}{\pi j^3}}
      \paren{\theta \sqrt{\frac{e}j}}^\ell
      \paren{ 1 - \frac{1+\ell}j }^{\frac{j-1-\ell}2}
      \\
      & \leq
      \sqrt[4]{\frac{2e^2}{\pi j^3}}
      \binom{j-1}\ell
      \paren{\theta \sqrt{\frac{e}j}}^\ell
      .
    \end{align*}
  We handle the final term, $\ell = j-1$, separately:
  \begin{align*}
     	\frac{\sqrt{2}}{2^{j/2}}  \binom{j-1}{\ell} \frac{(j-1-\ell)!}{\sqrt{j!} \paren{\frac{j-1-\ell}{2}}!} (\sqrt2\theta)^\ell 
	&= \frac{\sqrt{2}}{2^{j/2}}  \frac{1}{\sqrt{j!}} (\sqrt2\theta)^{j-1}
	\leq \frac{\sqrt{2}}{2^{j/2}} \frac1{(2\pi j)^{1/4}} \paren{ \sqrt{\frac{e}{j}} }^j(\sqrt{2}\theta)^{j-1} \\
	&=  \frac1{(2\pi j)^{1/4}} \sqrt{\frac{e}{j}} \paren{ \theta \sqrt{\frac{e}{j}} }^{j-1}
	 \leq
   \sqrt[4]{\frac{2e^2}{\pi j^3}}
   \binom{j-1}{\ell}
   \paren{ \theta \sqrt{\frac{e}{j}} }^\ell. 
  \end{align*}
  Therefore, we upper-bound the summation from \eqref{eq:Ajtheta-bound} term-by-term, and then further simplify bound by including additional non-negative terms in the summation:
    \begin{align*}
    A_{j, \theta}
      & \leq
      \sqrt[4]{\frac{2e^2}{\pi j^3}}
      \sum_{\text{odd} \ \ell=1}^{j-1} \binom{j-1}{\ell} \paren{ \theta \sqrt{\frac{e}{j}} }^\ell \\
      & \leq
      \sqrt[4]{\frac{2e^2}{\pi j^3}}
      \sum_{\ell=0}^{j-1} \binom{j-1}{\ell} \paren{ \theta \sqrt{\frac{e}{j}} }^\ell \\
      & =
      \sqrt[4]{\frac{2e^2}{\pi j^3}}
      \paren{ 1 + \theta \sqrt{\frac{e}{j}} }^{j-1}
      = B_{j,\theta}
    .
    \qedhere
	\end{align*}
\end{proof}

%
%
%
%

\begin{lemma} \label[lemma]{lemma:mills}
  For sufficiently large $k$,
  $\sqrt{2 \ln k - \ln(2\ln k)} \leq \theta_k \leq \sqrt{2 \ln k}$.
\end{lemma}
\begin{proof}
  Recall that $\theta_k$ is defined so that $\EE[\bx \sim \mnormalk]{\cubek(\bx)} = 0$.
  In other words, it is the median value of $\by := \max_{i \in [k]} |\bx_i|$, where $(\bx_1,\dotsc,\bx_k) \sim \normal^{\otimes k}$.
  Therefore, it suffices to show that for $l_k := \sqrt{2 \ln k - \ln(2\ln k)}$ and $u_k := \sqrt{2 \ln k}$, we have $\Pr[ \by < l_k ] \leq 1/2 \leq \Pr[ \by < u_k ]$.
  Note that for any $t\geq0$, $\Pr[\by < t] = (1 - \Pr_{\bx_1 \sim \normal}[|\bx_1| \geq t])^k$.
  Using the Mills ratio bound~\citep[see, e.g.,][Lemma 2 on page 175]{feller1968introduction} and $1-x \leq e^{-x}$ for all $x \in \R$,
  \begin{align*}
    \Pr[ \by < l_k ]
    & \leq 
    \left( 1 - \left(\frac1{l_k} - \frac1{l_k^3}\right) \sqrt{\frac2\pi} e^{-l_k^2/2} \right)^k
    \\
    & =
    \left( 1 - \frac1{\sqrt{2\ln k - \ln(2\ln k)}} \left(1 - o(1)\right) \sqrt{\frac2\pi} \cdot \frac{\sqrt{2\ln k}}k \right)^k
    \\
    & \leq \exp\left( -\left(1 - o(1)\right) \sqrt{\frac2\pi} \right) \leq \frac12
  \end{align*}
  by the choice of $l_k$ and assumption that $k$ is sufficiently large.
  Similarly (but now using $1-x \geq e^{-x/(1-x)}$ for $x < 1$),
  \begin{align*}
    \Pr[ \by < u_k ]
    & \geq 
    \left( 1 - \frac1{u_k} \sqrt{\frac2\pi} e^{-u_k^2/2} \right)^k
    \\
    & =
    \left( 1 - \frac1{\sqrt{\pi\ln k}} \cdot \frac1k \right)^k
    \\
    & \geq
    \exp\left( - \left(1 + o(1)\right) \frac1{\sqrt{\pi \ln k}} \right)
    \geq \frac12
    .
    \qedhere
  \end{align*}
\end{proof}


\subsection{Properties of $\dam{d, \tau}$ for sufficiently large $\tau$ (Proof of \Cref{prop:large-tau})}\label{asec:proof-trunc-prop}

We recall \Cref{prop:large-tau}.


\proplargetau*

\begin{proof}[Proof of \Cref{prop:large-tau}, part (i)]
	We first bound the low-degree Hermite coefficients of $\dam{d,\tau}(f)$.
	Fix some $J$ with $\abs{J} \leq d$.
  Then
	\begin{align*}
		\abs{\widetilde{\dam{d,\tau}[f]}(J)}
		&\leq \abs{\widetilde{\high{d}[f]}(J)} + \frac{1}{\sqrt{J!}} \abs{\EE{\high{d}[f](\bx) \indicator{\abs{\low{d}[f](\bx)} > \tau} H_J(\bx)]}} \\
		&\leq  \frac{1}{\sqrt{J!}} \norm[2]{\high{d}[f]} \sqrt{\EE{\indicator{\abs{\low{d}[f](\bx) > \tau}} H_J(\bx)^2 }} \\
		&\leq \frac{1}{\sqrt{J!}} \norm[2]{f} \pr{\low{d}[f](\bx) > \tau}^{1/4} \norm[4]{H_J} \\
		&\leq\frac{1}{\sqrt{J!}} \norm[2]{f} \exp\paren{-\frac{d}{8e} \paren{\frac{\tau}{\norm[2]{\low{d}[f]}}}^{2/d}} 3^{d/2} \sqrt{J!} \\
		&\leq  \norm[2]{f} \exp\paren{-\frac{d}{8e} \paren{\frac{\tau}{\rho}}^{2/d}} 3^{d/2} 
    .
	\end{align*}
	The first inequality follows from the linearity of the Hermite expansion and a triangle inequality.
	The second follows by Cauchy-Schwarz and the definition of $\high{d}(f)$.
	The third follows from $\norm[2]{\high{d}[f]}  \leq \norm[2]{f}$ and another application of Cauchy-Schwarz.
	The fourth uses \Cref{fact:hermite4} and \Cref{fact:poly-bound} (note that \eqref{eq:taudef} gives ${\tau}/{ \norm[2]{\low{d}[f]}} \geq e^d$, so \Cref{fact:poly-bound} can indeed be applied).
		
Now we consider the full Hermite expansion of $\low{d}(\dam{d,\tau}(f))$ and plug in $\tau$ to retrieve the claim:
	\begin{align*}
		\norm[2]{\low{d}[\dam{d,\tau}[f]]}^2
		&= \sum_{\abs{J} \leq d} \widetilde{\dam{d,\tau}[f]}(J)^2
		\leq k^d  \norm[2]{f}^2 \exp\paren{-\frac{d}{4e} \paren{\frac{\tau}{\rho}}^{2/d}} 3^{d}  \\
		&\leq (3k)^d  \norm[2]{f}^2 \exp\paren{-d \ln(3k) - \ln\paren{\frac{a^2 \norm[2]{f}^2}{\rho^2}}}
		= \frac{\rho^2 }{a^2}.
	\end{align*}
In the first inequality, we used the fact that the number of $k$-dimensional multi-indices $J$ with $|J| \leq d$ is at most $k^d$ for $d \geq 2$.
\end{proof}
	
\begin{proof}[Proof of \Cref{prop:large-tau}, part (ii)]
	We have
	\begin{align*}
		\norm[1]{\dam{d,\tau}[f] - f}
		&\leq \norm[1]{f - \high{d}[f]}  + \norm[1]{\high{d}[f] \indicator{\abs{\low{d}[f]} > \tau}} \\
		&\leq \norm[1]{\low{d}[f]} + \norm[2]{\high{d}[f]} \sqrt{\pr{\abs{\low{d}[f]} > \tau}} \\
		&\leq \norm[2]{\low{d}[f]} + \norm[2]{f} \sqrt{\pr{\abs{\low{d}[f]} > \tau}} \\
		&\leq \norm[2]{\low{d}[f]} + \norm[2]{f} \pr{\abs{\low{d}[f]} > \tau}^{1/4},
	\end{align*}
	where the first inequality is by the triangle inequality and the definition of $\dam{d,\tau}$, the second is Cauchy-Schwarz, and the third is monotonicity of norms and $\norm[2]{\high{d}[f]} \leq \norm[2]{f}$.
	We once again use \Cref{fact:poly-bound} and $\tau$ to obtain
	\begin{align*}
		\norm[1]{\dam{d,\tau}[f] - f}
		&\leq \rho + \norm[2]{f} \exp\paren{-\frac{d}{8e} \paren{4e \log(3k) + \frac{8e}{d} \log\paren{\frac{a\norm[2]{f}}{\rho}}}} \\
		&\leq \rho + \frac\rho{a} \leq 2 \rho.\qedhere
	\end{align*}
\end{proof}


\subsection{Proof of exponential decay of $\tau_i$ for \Cref{lemma:resilience}}\label{asec:proof-tau-bound}

\begin{fact}\label[fact]{fact:tau-bound}
	For any fixed $i \geq 1$ and \[\tau_i := \frac{\norm[2]{\low{d}[f_0]}}{4^{(i-1)d}} \paren{4e \ln(3k) + \frac{8e}{d} \ln\paren{\frac{4^{id} \norm[2]{f_{i-1}}}{\norm[2]{\low{d}[f_0]}}}}^{d/2}\] from \eqref{eq:tau}, if $\norm[\infty]{f_{i-1}} \leq \frac43$, then $\tau_i \leq \frac\alpha{3 \cdot 2^{i}}$ for $\alpha =  \norm[2]{\low{d}[f_0]}^{0.996} (72\ln k)^{d/2}.$ In addition, $\tau_1 \geq\norm[2]{\low{d}[f_0]}$.
\end{fact}
\begin{proof}	
	We first consider the case where $i = 1$. 
	A sufficiently large choice of $k$ yields the following:		\begin{align*}
		\tau_1 
		&=\norm[2]{\low{d}[f_0]} \paren{4e \ln 3k + 8e \ln 4 + \frac{8e}d \ln \paren{\frac1{\norm[2]{\low{d}[f_0]}}}}^{d/2} & [\norm[2]{f_0} = 1] \\
		&\leq\norm[2]{\low{d}[f_0]} \paren{11 \ln k + \frac{1000e}{\norm[2]{\low{d}[f_0]}^{1/125d}}}^{d/2}   & [4e \ln 3k + 8e\ln 4 \leq 11\ln k;  \ \ln x \leq x] \\
		&\leq  \norm[2]{\low{d}[f_0]} \paren{{\frac{12}{ \norm[2]{\low{d}[f_0]}^{1/125d}} \ln k}}^{d/2} & [\forall x \geq 1, \ 11\ln k + 1000ex \leq 12 x \ln k] \\
		&= \norm[2]{\low{d}[f_0]}^{0.996} (12\ln k)^{d/2} 
		\leq  \frac\alpha6.
		\end{align*}
	Observe that $\tau_1 \geq \norm[2]{\low{d}[f_0]}$ for sufficiently large $k$, because the base of the exponent will always be at least 1.

For fixed $i \geq 2$, we prove $\tau_i \leq \frac{\alpha}{e \cdot 2^{i}}$ by bounding $\frac{\tau_i}{\tau_1}$. 
Using the assumption that $\norm[2]{f_{i-1}} \leq \norm[\infty]{f_{i-1}} \leq \frac43$,
	\begin{align*}
		\frac{\tau_i}{\tau_1}
		 &= \frac1{4^{(i-1)d}}\paren{\frac{4e\ln(3k) + \frac{8e}{d} \ln(4^{id} \norm[2]{f_{i-1}}) - \ln  \norm[2]{\low{d}[f_0]}}{4e\ln(3k) + \frac{8e}{d} \ln(4^{d} \norm[2]{f_{0}}) - \ln  \norm[2]{\low{d}[f_0]}}}^{d/2} \\
		&\leq \frac1{4^{(i-1)d}} \paren{\frac{\ln 4^{(i+1)d}}{\ln 4^{d}}}^{d/2} 
		\leq \paren{\frac{i+1}{16^{i-1}}}^{d/2}
		\leq \frac{1}{4^{i d / 2}} \leq \frac1{2^{i-1}}.\qedhere
	\end{align*}

\end{proof}


\subsection{Proof of convergence of $f_i$'s in \Cref{lemma:resilience}}\label{asec:limit}

The proof of \Cref{lemma:resilience} constructs
a sequence of functions $f_0, f_1, \dots \in L^2(\mnormalk)$ with the following properties for any $a$, $b$, and $c$ having $b \geq 4$ and $ c\leq 2$:
\begin{enumerate}
	\item For all $i$, $\norm[1]{f_{i+1} - f_{i}} \leq \frac{a}{b^{i}}$.
	\item For all $i$, $\norm[\infty]{f_i} \leq c$.
	\item $\lim_{i \to \infty} \norm[2]{\low{d}(f_i)} = 0.$
\end{enumerate}
We now prove that such a sequence has a limit in $L^2(\mnormalk)$ with the desired properties, as given in the following proposition.


\begin{prop}\label[proposition]{prop:limit-argument}
	For the sequence described above, there exists some $f^* \in L^2(\mnormalk)$ such that $\low{d}(f^*) = 0$, $\norm[\infty]{f^*} \leq c$, and $\norm[1]{ f^*-f_i} \leq \frac{2a}{b^{i}}$ for all $i$.
\end{prop}

Towards the proof of \Cref{prop:limit-argument}, we first show that properties (1) and (2) imply an additional property about $L^2$ distances between iterates.

\begin{lem}\label[lemma]{lemma:l1-l2}
	For all $i$, $\norm[2]{f_{i+1} - f_{i}} \leq \sqrt{\frac{2 ac}{b^i}}$.
\end{lem}
\begin{proof}
	By the triangle inequality we have $\norm[\infty]{f_{i+1} - f_{i}}\leq 2c$, and from this the bound is immediate from Holder's inequality:
	\begin{align*}
		\norm[2]{f_{i+1} - f_{i}}
		&\leq \sqrt{\norm[\infty]{f_{i+1} - f_i} \norm[1]{f_{i+1} - f_i} }
		\leq \sqrt{2c \cdot \frac{a}{b^i}}.\qedhere
	\end{align*}\end{proof}

The following is immediate from \Cref{lemma:l1-l2}\ignore{ exponential rate of convergence of the $f_i$ iterates in $L^2(\mnormalk)$} and the fact that $L^2(\mnormalk)$ is complete (because it is a Hilbert space).

\begin{cor}
	The sequence $f_0, f_1, \dots$ is a Cauchy sequence in $L^2(\mnormalk)$ and converges to some $f^* \in L^2(\mnormalk)$.
\end{cor}

Before completing the proof of \Cref{prop:limit-argument}, we recall the following topological fact concerning functional spaces $L^2$ and $L^\infty$.

\begin{lem}\label[lemma]{lemma:l2-linfy}
    For any probability measure $\mu$ on $\mathbb{R}^k$ and
    any $\alpha>0$,
    \[I_\alpha := \left\{f\in L^2(\mu) : \norm[\infty]{f} \leq \alpha\right\} \ \text{is a closed set in $L^2(\mu)$} . \]
\end{lem}
  \begin{proof}
  Consider any functional sequence $(f_n)_{n \in \N}$ in $I_\alpha$ such that $f_n\to f$ in $L^2(\mu)$ as $n\to\infty$.
  It is clear that the limit $f$ belongs to $L^2(\mu)$ since $L^2(\mu)$ is, by itself, closed.
  Thus, it suffices to prove that 
$\pr[\bx\sim\mu]{\abs{f(x)}\leq \alpha}=1$. 
Fix any $\varepsilon > 0$ and $n \in \N$.
  \begin{align*}
    \pr[\bx\sim\mu]{|f(\bx)| > \alpha + \varepsilon}
    & = \pr{|f(\bx)| > \alpha + \varepsilon \ \wedge \ |f_n(\bx)| \leq \alpha} \\
    & \leq \pr{|f(\bx) - f_n(\bx)| > \varepsilon}\\
    & \le \frac{1}{\varepsilon^2}\int_{\mathbb{R}^k} |f(x)-f_n(x)|^2\mathrm{d}\mu(x)= \frac{\norm[2]{f - f_n}^2}{\varepsilon^2}.
  \end{align*}
  The final step follows from Chebyshev's inequality.
  By assumption, we have $\norml[2]{f - f_n} \to 0$ as $n \to \infty$.
  For every $\varepsilon>0$, we have
  $\pr{|f(\bx)| > \alpha + \varepsilon} = 0$. 
  Hence, $\norml[\infty]{f} \leq \alpha$ and $f\in I_\alpha$.
\end{proof}




\begin{lem}\label[lemma]{lemma:subset-to-prop}
	$f^*$ satisfies the properties given in \Cref{prop:limit-argument}.
\end{lem}
\begin{proof}
Let $B_i = \{g \in L^2(\mnormalk): \ \norm[2]{g - f_i} \leq 2\sqrt{ \frac{2ac}{b^i}}\}$ be the closed set containing all functions in a small $L^2$-ball around the $i$-th iterate.
Note that $B_{i+1} \subset B_i$ for all $i \geq 0$ and that $ \bigcap_{i \geq 0} B_i = \{f^*\}$.
We prove each property of \Cref{prop:limit-argument}.

\begin{enumerate}

	\item 
	Suppose that $\norm[2]{\low{d}(f^*)} \geq \eps$ for any fixed $\eps > 0$. 
	For any sufficiently large $i$, 
	\[\norm[2]{f^* - f_i} \geq \norm[2]{\low{d}(f^*) - \low{d}(f_i)} \geq \norm[2]{\low{d}(f^*)} - \norm[2]{\low{d}(f_i)} \geq \eps - \frac{\eps}{2} = \frac{\eps}{2}.\]
	This would mean that exists some $i'$ such that $\norm[2]{f^* - f_{i'}} \geq 2\sqrt{\frac{2ac}{b^{i'}}}$, but then $f^*$ would lie outside  $B_{i'}$, which is a contradiction.
	
	\item
	Let $I = \{g \in L^2(\mnormalk): \ \norm[\infty]{g} \leq c\}$.
	By \Cref{lemma:l2-linfy} (with $\mu=\mnormalk$), $I$ is closed in $L^2(\mnormalk)$  and $f_0, f_1, \dots$ is a sequence in $I$ with limit $f^* \in L^2(\mnormalk)$, we must have that $f^* \in I$ as well.
    Thus, $\norm[\infty]{f^*} \leq c$.
		
		\item
	Fix any $i \geq 0$. 
	Choose some $i' > i$ such that $b^{i'} \geq \frac{18 b^{2i}c}{a}$.
	Because $f^* \in B_{i'}$, it follows that $\norm[1]{f_{i'} - f^*} \leq \norm[2]{f_{i'} - f^*} \leq 2\smash{\sqrt{\frac{2ac}{b^{i'}}}} \leq \frac{2a}{3b^i}$. Thus,
	\[\norm[1]{f^* - f_i} \leq \norm[1]{f_{i'}- f_{i}} + \norm[1]{f^* - f_{i'}} \leq \sum_{\iota = i}^{i' - 1} \frac{a}{b^{\iota}} + \frac{2a}{3b^{i}} \leq \frac{a}{b^{i}} \sum_{\iota = 0}^\infty \frac1{4^\iota} + \frac{2a}{3b^{i}}= \frac{2a}{b^{i}}.\qedhere \]
\end{enumerate}
\end{proof}

\section{Supporting proof for \Cref{sec:lb-weak-learn}}

\subsection{Existence of a hard-to-weak-learn intersection of halfspaces (Proof of \Cref{cor:ds})}\label{asec:weak-learn-cor}
 
\corollaryds*

 \begin{proof}
	In the proof of \Cref{thm:ds21-lb}, $\dactual$ is a distribution which is supported on intersections of finitely many halfspaces.
	In more detail, for $\Lambda = q^{100} \ln 2$ and some $M \gg \Lambda$ (the exact value is not important for our purposes), a draw of $\boldf \sim \dactual$ is defined in the proof of Theorem~2 of \cite{ds21} to be an intersection of $H_{\boldf} \leq M$ halfspaces from a fixed collection $\{h_1,\dots,h_M\}$, where each halfspace $h_i$ is independently included in the intersection with probability $\frac{\Lambda}{M}$. Note that the expected number of halfspaces included in $\boldf$ is $\E{H_{\boldf}} = \Lambda$.

	We define $\dist$ to be the conditional distribution of $\dactual$ conditioned on $\boldf\sim \dactual$ being an intersection of at most $q^{101}$ halfspaces.
	By Markov's inequality, we have that $H_{\boldf} \leq q^{101} \ln 2 \leq q^{101}$ with probability at least $1 - {\frac 1 q}$.
	We bound the expected accuracy of the classifer $h$ returned by $\mathcal{A}$ for random $\boldf \sim \dist$ by comparing it to the expected error of a random $\boldf \sim \dactual$:
\begin{align*}
		\pr[\boldf \sim \dist, \mathcal{A}, \bx]{h(\bx) \neq \boldf(\bx)}
		&= \pr[\boldf \sim \dactual, \mathcal{A}, \bx]{h(\bx) \neq \boldf(\bx) \mid H_{\boldf} \leq  q^{101}} \\
		&\geq \pr[\boldf \sim \dactual, \mathcal{A}, \bx]{h(\bx) \neq \boldf(\bx), H_{\boldf} \leq  q^{101}} \\
		&\geq \pr[\boldf \sim \dactual, \mathcal{A}, \bx]{h(\bx) \neq \boldf(\bx)} - \pr[\boldf \sim \dactual, \mathcal{A}, \bx]{H_{\boldf} >  q^{101}} \\
		&\geq \frac12 - \frac{O(\log q)}{\sqrt{m}} - \frac1q
		\geq \half - \frac{O(\log q)}{\sqrt{m}},
\end{align*}
	where in the last line we used that ${\frac 1 {q}}=\frac{O(\log q)}{\sqrt{m}}$ (with room to spare) since $q \geq m$.
	\end{proof}


\section{Supporting lemmas and proofs for \Cref{sec:ganzburg}}\label{asec:ganzburg-proofs}
In what follows we describe our approach to strengthen the SQ lower bounds from \Cref{sec:lb-weak-learn}; the lower bounds from \Cref{sec:lb-resilience} can be similarly strengthened in an entirely analogous fashion.
Recall the intersection $k+1$ halfspaces over $\R^{m+1}$ obtained by taking the intersection of
\begin{itemize}
  \item the $k$ halfspaces identified in \Cref{lemma:random-intersect} over $\R^m$; and 
   
  \item an origin-centered halfspace orthogonal to the $(m+1)$-st coordinate basis vector.
\end{itemize}

This intersection of $k+1$ halfspaces $f: \R^{m+1} \to \bits$ can be written as $f(x_1,\dots,x_{m+1}) = f_1(x_1,\dots,x_m) \wedge f_2(x_{m+1})$, where $f_1: \R^m \to \bits$ is the intersection of $k$ halfspaces given in \Cref{lemma:random-intersect}, $f_2: \R \to \bits$ is the $\sign(\cdot)$ function which outputs 1 on an input $z$ iff $z > 0$, and the ``$\wedge$'' of two values from $\bits$ is 1 iff both of them are 1.


The following lemma gives a lower bound on the approximate degree of $f$ in terms of the approximate degrees of $f_1$ and $f_2$.


\newcommand\bz{{\mathbf{z}}}


\begin{lem} \label[lemma]{lem:product}
  Let $\bz_1$ and $\bz_2$ be independent random variables in $\R^{n_1}$ and $\R^{n_2}$, respectively.
  Fix any $\eps>0$ and $g_i \colon \R^{n_i} \to \bits$ for $i \in \{1,2\}$ with $c := \min\{ \Pr_{\bz_1}[g_1(\bz_1)=1], \Pr_{\bz_2}[g_2(\bz_2)=1]  \} > 0$.
  Then the function $g \colon \R^{n_1 + n_2} \to \bits$ defined by $g(z_1,z_2) := g_1(z_1) \wedge g_2(z_2)$ has $L^1$ $(c\eps)$-approximate degree at least $\max\{d_1,d_2\}$ (with respect to the joint distribution of $(\bz_1,\bz_2)$), where $d_i$ is the $L^1$ $\eps$-approximate degree of $g_i$ (with respect to the marginal distribution of $\bz_i$).
\end{lem}

\begin{proof}
  Assume without loss of generality that $d_1 \geq d_2$.
  For a $\bits$-valued function $h$, let $h' = {\frac {h + 1} 2}$ (so $h'$ is the $\{0,1\}$-valued version of $h$). We observe that $c = \min\{ \mathbb{E}_{\bz_1}[g'_1(\bz_1)], \mathbb{E}_{\bz_2}[g'_2(\bz_2)] \},$
  and that $d_i$ is the $L^1$ $(\eps/2)$-approximate degree of $g'_i$.

  Let $d$ be the $L^1$ $({c\eps/2})$-approximate degree of ${g'}$, and let $p'$ be a degree-$d$ polynomial over $\R^{n_1+n_2}$ satisfying $\mathbb{E}_{\bz_1,\bz_2}[|{g'}(\bz_1,\bz_2)-{p'}(\bz_1,\bz_2)|] \leq {c\eps/2}$.
  For this polynomial $p$,
  \begin{align*}
    \min_{{z_2 \in \R^{n_2} : {g'_2}(z_2)=1}}
    \mathbb{E}_{\bz_1}[\abs{{g'}(\bz_1, z_2) - {p'}(\bz_1, z_2)}]
    & \leq \mathbb{E}_{\bz_1,\bz_2}\left[\abs{{g'}(\bz_1, \bz_2) - {p'}(\bz_1, \bz_2)} \cdot \frac{{g'_2}(\bz_2)}{\mathbb{E}_{\bz_2}[{g'_2}(\bz_2)]} \right] \\
    & \leq \frac{\mathbb{E}_{\bz_1,\bz_2}[\abs{{g'}(\bz_1, \bz_2) - {p'}(\bz_1, \bz_2)}]}{c}
    \leq \frac\eps2 .
  \end{align*}
  So there exists $z_2 \in \R^{n_2}$ such that
  \begin{equation*}
   \mathbb{E}_{\bz_1}[\abs{{g'}(\bz_1, z_2) - {p'}(\bz_1, z_2)}] = \mathbb{E}_{\bz_1}[\abs{{g'_1}(\bz_1) - {p'}(\bz_1, z_2)}] \leq \frac\eps2 .
  \end{equation*}
 Letting $p=2p'-1$, since $g=2g'-1$, 
 there exists $z_2 \in \R^{n_2}$ such that
 \begin{equation*}
  \mathbb{E}_{\bz_1}[\abs{g(\bz_1, z_2) - p(\bz_1, z_2)}] = 2 \, \mathbb{E}_{\bz_1}[\abs{g'_1(\bz_1) - p'(\bz_1, z_2)}] \leq \eps .
 \end{equation*}
   Since $p(\cdot,z_2)$ is a polynomial over $\R^{n_1}$ of degree at most $d$, it follows that the $L^1$ ${(\eps/2)}$-approximate degree of ${g'_1}$ is at most $d$.
   Hence $d \geq d_1 =\max\{d_1,d_2\}$.
   Since the $L^1$ $(c\eps/2)$-approximate degree of $g'$ (which is $d$) is the same as the $L^1$ $(c\eps)$-approximate degree of $g$, the lemma is proved.
\end{proof}

By \Cref{lemma:random-intersect}, the $L^1$ ${\half}$-approximate degree of $f_1$ is at least $\Omega(\log(k)/\log\log k)$, and hence so is its $L^1$ $(4\eps)$-approximate degree (for $\eps \leq 1/8$).
A lower bound on the $L^1$ $(4\eps)$-approximate degree of $f_2$ is given by the following result of \cite{ganzburg02}.

\begin{lem}  \label[lemma]{lem:ganzburg}
  For any $\eps>0$, the $L^1$ $\eps$-approximate degree of the $\sign(\cdot)$ function is $\Omega(1/\eps^2)$.
\end{lem}

\Cref{lem:ganzburg} (presented as Corollary~B.1 of \cite{dkpz21}) is a direct consequence of Theorem~1 of \cite{ganzburg02} and Theorem~4 of \cite{vaaler1985some}.

We are now almost ready to apply \Cref{lem:product} to our intersection of $k+1$ halfspaces obtained via $f_1$ (from \Cref{lemma:random-intersect}) and $f_2$ (the sign function).
We just need to ensure that each of $f_1$ and $f_2$ takes value $+1$ with sufficient probability.
First, observe that $f_1$ satisfies $\Pr_{\bx \sim {\cal N}(0,I_m)}[f_1(\bx)=1] \geq 1/4$, since otherwise the $1/2$-approximate degree of $f$ would be zero, as witnessed by the constant $-1$ function.
Moreover, $\Pr_{\bx \sim {\cal N}(0,1)}[f_2(\bx)=1] = 1/2$ by symmetry of ${\cal N}(0,1)$.
So, we have established that both $f_1$ and $f_2$ take value $+1$ with probability at least $1/4$, and that also they have $(4\epsilon)$-approximate degrees $\Omega(\log(k)/\log\log k)$ and $\Omega(1/\epsilon^2)$, respectively.
Therefore, \Cref{lem:product} implies a lower bound on the $L^1$ $\eps$-approximate degree of $f$, as stated in the following lemma.
 

%


\begin{lem}\label[lemma]{lem:intersect-ganzburg}
For any $k = 2^{O(n^{0.245})}$ and any $\eps > 0$, there is an intersection of $k+1$ halfspaces $f \colon \R^{m+1} \to \bits$ with $L^1$ $\eps$-approximate degree $\Omega(\frac{\log k}{\log\log k} + \frac1{\eps^2})$, {where $m = O(n^{0.49})$}.
\end{lem}

\Cref{lem:intersect-ganzburg} and \Cref{thm:dkpz-sq-lb} together imply \Cref{thm:alt-cube-ganzburg}.  

\theoremaltcubeganzburg*

\end{document}